\documentclass[twoside,11pt]{article}

\usepackage{jmlr2e}
\usepackage{amsmath,dsfont}
\usepackage{algorithm}
\usepackage{algpseudocode}
\usepackage{syntonly,bm}
\usepackage{amsfonts,color}
\input{mysymbol.sty}
\usepackage{subfig}
\usepackage{natbib}
\usepackage{xcolor}
\usepackage{tabu}

\def \bV {{\bf V}}

\def \bV {{\bf V}}

\def \bx {{\bf x}}

\def \bbeta {{\boldsymbol \beta}}

\usepackage{lipsum}

\makeatletter
\def\blfootnote{\xdef\@thefnmark{}\@footnotetext}
\makeatother

\firstpageno{1}

\usepackage{lastpage}
\jmlrheading{19}{2018}{1-\pageref{LastPage}}{1/18; Revised
11/18}{12/18}{18-030}{Yanning Shen and Tianyi Chen and Georgios B. Giannakis}
\ShortHeadings{Random Feature-based Online MKL in 
Environments with Unknown Dynamics}{Shen and Chen and Giannakis}

\begin{document}

\title{Random Feature-based Online Multi-kernel Learning\\ in 
Environments with Unknown Dynamics}

\author{\name Yanning Shen \email shenx513@umn.edu 
       \AND
       \name Tianyi Chen \email chen3827@umn.edu 
        \AND
       \name Georgios B. Giannakis \email georgios@umn.edu ~\\
       \addr Department of Electrical and Computer Engineering, University of Minnesota\\
       Minneapolis, MN, 55455, USA}

\editor{Karsten Borgwardt}

\maketitle

\begin{abstract}
Kernel-based methods exhibit well-documented performance in various nonlinear learning tasks. Most of them rely on a preselected kernel, whose prudent choice presumes task-specific prior information. Especially when the latter is not available, multi-kernel learning has gained popularity thanks to its flexibility in choosing kernels from a prescribed kernel dictionary. Leveraging the random feature approximation and its recent orthogonality-promoting variant, the present contribution develops a scalable multi-kernel learning scheme (termed Raker) to obtain the sought nonlinear learning function `on the fly,' first for static environments. To further boost performance in dynamic environments, an adaptive multi-kernel learning scheme (termed AdaRaker) is developed. {AdaRaker accounts not only for data-driven learning of kernel combination, but also for the unknown dynamics. Performance is analyzed in terms of both static and dynamic regrets. AdaRaker is uniquely capable of tracking nonlinear learning functions in environments with unknown dynamics, and with with analytic performance guarantees.}  
Tests with synthetic and real datasets are carried out to showcase the 
effectiveness of the novel algorithms.\footnote{Preliminary results in this paper were presented in part at the 2018 International Conference on Artificial Intelligence and Statistics \citep{shen2018aistats}.}
\end{abstract}

\begin{keywords}
Online learning, reproducing kernel Hilbert space, multi-kernel learning, random features, dynamic and adversarial environments.
\end{keywords}

\section{Introduction}
 
Function approximation emerges in various learning tasks such as regression, classification, clustering, dimensionality reduction, as well as reinforcement learning \citep{scholkopf2002,shawe2004,dai2017}. 
{Among them, the emphasis here is placed on supervised functional learning tasks: given samples $\{(\bbx_1, y_1),\ldots, (\bbx_T, y_T)\}_{t=1}^T$ with $\bbx_t\in \mathbb{R}^d$ and $y_t\in \mathbb{R}$, the goal is to find a function $f(\cdot)$ such that the discrepancy between each pair of $y_t$ and $f(\bbx_t)$ is minimized. Typically, such discrepancy is measured by a cost function ${\cal C}(f(\bbx_t),y_t)$, which requires to find $f(\cdot)$ minimizing $\sum_{t=1}^T{\cal C}(f(\bbx_t),y_t)$. While this goal is too ambitious to achieve in general, the problem becomes tractable when $f(\cdot)$ is assumed to belong to a reproducing kernel Hilbert space (RKHS) induced by a kernel \citep{scholkopf2002}. Comparable to deep neural networks, functions defined in RKHS can model highly nonlinear relationship, and thus kernel-based methods have well-documented merits for principled function approximation. 
Despite their popularity, most kernel methods rely on a single pre-selected kernel. Yet, multi-kernel learning (MKL) is more powerful}, thanks to its data-driven kernel selection from a given dictionary; see e.g.,  \citep{shawe2004,rakotomamonjy2008,cortes2009,gonen2011}, and \citep{bazerque2013nonparametric}. 

{In addition to the attractive representation power that can be afforded by kernel methods,} several learning tasks are also expected to be performed in an online fashion. 
Such a need naturally arises when the data arrive sequentially, such as those in online spam detection \citep{ma2009}, and time series prediction \citep{richard2009}; or, when the sheer volume of data makes it impossible to carry out data analytics in batch form \citep{kivinen2004}.  
This motivates well online kernel-based learning methods that inherit the merits of their batch counterparts, while at the same time allowing efficient online implementation.
Taking a step further, the optimal function may itself change over time in environments  with \emph{unknown dynamics}. This is the case when the function of interest e.g., represents the state in brain graphs, or, captures the temporal processes propagating over time-varying networks. Especially when variations are due to adversarial interventions, the underlying dynamics are unknown. Online kernel-based learning 
in such environments remains a largely uncharted territory \citep{kivinen2004,hoi2013}. 

In accordance with these needs and desiderata, the \emph{goal} of this paper is an algorithmic pursuit of scalable online MKL in environments with unknown dynamics, along with their associated performance guarantees.   
Major challenges come from two sources: i) the well-known ``curse'' of dimensionality in kernel-based learning; and, ii) the defiance of tracking unknown time-varying functions without future information. 
Regarding i), the representer theorem renders the size of kernel matrices to grow quadratically with the number of data \citep{wahba1990spline}, thus the computational complexity to find even a \emph{single} kernel-based predictor is cubic. 
Furthermore, storage of past data causes memory overflow in large-scale learning tasks such as those emerging in e.g., topology identification of social and brain networks \citep{shen2016tmi,shen2017tsp,shen2018dsw}, which makes kernel-based methods less scalable relative to their linear counterparts.  
For ii), most online learning settings presume time invariance or slow dynamics, where an algorithm achieving sub-linear regret incurs on average ``no-regret'' relative to the \emph{best static} benchmark. 
Clearly, designing online schemes that are comparable to the \emph{best dynamic} solution is appealing though formidably challenging without knowledge of the dynamics \citep{kivinen2004}. 

\subsection{Related works}

To put our work in context, we review prior art from the following two aspects.

\vspace{0.2cm}
	\noindent{\bf Batch kernel methods}. 
	Kernel methods are known to suffer from the growing dimensionality in large-scale learning tasks \citep{shawe2004}. 
	Major efforts have been devoted to scaling up kernel methods in batch settings. Those include approaches to approximating the kernel matrix using low-rank factorizations \citep{williams2001,sheikholeslami2016}, whose performance was analyzed in \citep{cortes2010}. Recently, random feature (RF) based function estimators have gained popularity since the work of \citep{rahimi2007} and \citep{dai2014}, whose variance has been considerably reduced through an \emph{orthogonality promoting} RF modification \citep{felix2016}.  
These approaches assume that the kernel is known, a choice crucially dependent on domain knowledge. 
Enabling kernel selection, several MKL-based approaches have emerged, see e.g.,  \citep{lanckriet2004,rakotomamonjy2008,bach2008,cortes2009,gonen2011}, and their performance gain has been documented relative to their single kernel counterparts. 
However, the aforementioned methods are designed for batch settings, and are either intractable or become less efficient in online setups. When the sought functions vary over time and especially when the dynamics are unknown (as in adversarial settings), batch schemes fall short in tracking the optimal function estimators.  		

\vspace{0.2cm}	
	\noindent {\bf Online (multi-)kernel learning}. Tailored for streaming large-scale datasets, online kernel-based learning methods have gained due popularity. 
   To deal with the growing complexity of online kernel learning, successful attempts have been made to design \emph{budgeted} kernel learning algorithms, including techniques such as support vector removal \citep{kivinen2004,dekel2008}, and support vector merging \citep{wang2012}. {Maintaining an affordable budget, online multi-kernel learning (OMKL) methods have been reported for online classification \citep{jin2010,hoi2013,sahoo2016}, and regression \citep{sahoo2014online,lu2018}.} 
   Devoid of the need for budget maintenance, {online kernel-based learning algorithms based on RF approximation \citep{rahimi2007} have been developed in \citep{lu2016large,bouboulis2017,ding2017large}, but only with a single pre-selected kernel.}  
More importantly, existing kernel-based learning approaches implicitly presume a \emph{static} environment, where the benchmark is provided through the best static function (a.k.a. static regret) \citep{shalev2011}. 
However, static regret is not a comprehensive metric for dynamic settings, where the optimal kernel also varies over time and the dynamics are generally unknown as with adversarial settings.

 \subsection{Our contributions}

The present paper develops an adaptive online MKL algorithm, capable of learning a nonlinear function from sequentially arriving data samples. 
Relative to prior art, our contributions can be summarized as follows. 

\textbf{c1)} For the first time, RFs are employed for scalable online MKL 
tackled by a weighted combination of advices from an ensemble of experts - an innovative
cross-fertilization of online learning to MKL. Performance of the resultant algorithm (abbreviated as \textbf{Raker}) is benchmarked by the best time-invariant function approximant via static regret analysis.

\textbf{c2)} A novel adaptive approach (termed \textbf{AdaRaker}) is introduced for 
scalable online MKL in environments with unknown dynamics. AdaRaker is a hierarchical ensemble learner with scalable RF-based modules that provably yields sub-linear dynamic regret, so long as the accumulated variation grows sub-linearly with time.

\textbf{c3)} The novel algorithms are compared with competing alternatives for online nonlinear regression on both synthetic and real datasets. The tests corroborate that Raker and AdaRaker exhibit attractive performance in both accuracy and scalability. 

\noindent\textbf{Outline}. 
Section~\ref{sec:pre} presents preliminaries, and states the problem. Section~\ref{sec3} develops the Raker for online MKL in static environments, and Section~\ref{sec4} develops its adaptive version for online MKL in environments with unknown dynamics.  Section \ref{sec:test} reports numerical tests with both synthetic and real datasets, while conclusions are drawn in Section \ref{sec:conc}.

\vspace{0.1cm}
\noindent\textbf{Notation}. Bold uppercase (lowercase) letters will denote
matrices (column vectors), while $(\cdot)^{\top}$ stands for
vector and matrix transposition, and $\|\mathbf{x}\|$ denotes the $\ell_2$-norm of a vector $\mathbf{x}$. Inequalities for vectors $\mathbf{x} > \mathbf{0}$, and the projection operator $[\mathbf{a}]^+:=\max\{\mathbf{a},\mathbf{0}\}$ are defined entrywise. Symbol $\dag$ represents the Hermitian operator, while the indicator function $\mathds{1}_{\{A\}}$ takes value $1$ when the event $A$ happens, and $0$ otherwise. 
$\mathbb{E}$ denotes the expectation, while $\langle \cdot, \cdot \rangle$ and $\langle \cdot, \cdot \rangle_{\cal H}$ the vector inner product in Euclidian and  Hilbert space respectively.

\section{Preliminaries and Problem Statement}\label{sec:pre}
This section reviews briefly basics of kernel-based learning, to introduce notation and the needed background for our novel online MKL schemes.

Given samples $\{(\bbx_1, y_1),\ldots, (\bbx_T, y_T)\}_{t=1}^T$ with  $\bbx_t\in \mathbb{R}^d$ and  $y_t\in \mathbb{R}$, the function approximation task is to find a function $f(\cdot)$ such that $y_t=f(\bbx_t)+e_t$, where $e_t$ denotes an error term representing noise or un-modeled dynamics. 
It is supposed that $f(\cdot)$ belongs to a reproducing kernel Hilbert space (RKHS), namely $\mathcal{H}:=\{f|f(\bbx)=\sum_{t=1}^{\infty} \alpha_t\kappa(\bbx, \bbx_t)\}$, 
where $\kappa(\bbx,\bbx_t): \mathbb{R}^d\times \mathbb{R}^d\rightarrow \mathbb{R}$ is a symmetric positive semidefinite basis (so-termed kernel) function, which measures the similarity between $\bbx$ and $\bbx_t$. 
Among the choices of $\kappa$ specifying different bases, a popular one is the Gaussian given by $\kappa(\bbx, \bbx_t):= \exp[-\|\bbx-\bbx_t\|^2/(2\sigma^2)]$. A kernel is reproducing if it satisfies $\langle \kappa(\bbx, \bbx_{t}), \kappa(\bbx,\bbx_{t'})\rangle_{\cal H}=\kappa(\bbx_t,\bbx_{t'})$, which in turn induces the RKHS norm $\|f\|_{\mathcal{H}}^2:=\sum_{t}\sum_{t'} \alpha_t\alpha_{t'}\kappa(\bbx_t,\bbx_{t'})$. Consider the optimization problem 
\begin{align}\label{opt0}
	 \min_{f\in \mathcal{H}}~\frac{1}{T}\sum_{t=1}^T{\cal C}(f(\bbx_t),y_t)+\lambda\Omega\left(\|f\|_{\mathcal{H}}^2\right)
\end{align}
where depending on the application, the cost function ${\cal C}(\cdot,\cdot)$ can be selected to be, e.g., the least-squares (LS), the logistic or the hinge loss; $\Omega(\cdot)$ is an increasing function; and, $\lambda>0$ is a regularization parameter that controls overfitting. According to the representer theorem, the optimal solution of \eqref{opt0} admits the finite-dimensional form, given by \citep{wahba1990spline}
\begin{align}\label{eq:sol0}
	\hat{f}(\bbx)=\sum_{t=1}^T\alpha_t \kappa(\bbx,\bbx_t):=\bm{\alpha}^{\top}\mathbf{k}(\mathbf{x})
\end{align}
where $\bm{\alpha}:=[\alpha_1,\ldots,\alpha_T]^{\top}\!\in\mathbb{R}^T$ collects the combination coefficients, and the $T\times 1$ kernel vector is  $\mathbf{k}(\mathbf{x}):=[\kappa(\bbx,\bbx_1),\ldots,\kappa(\bbx,\bbx_T)]^{\top}\!$.
Substituting \eqref{eq:sol0} into the RKHS norm, we find $\|f\|_{\mathcal{H}}^2:=\sum_{t}\sum_{t'} \alpha_t\alpha_{t'}\kappa(\bbx_t,\bbx_{t'})=\bm{\alpha}^{\top}\mathbf{K}\bm{\alpha}$, where the $T\times T$ kernel matrix $\mathbf{K}$ has entries $[\mathbf{K}]_{t,t'}:=\kappa(\bbx_t,\bbx_{t'})$; thus, the functional problem \eqref{opt0} boils down to a $T$-dimensional optimization over $\bm{\alpha}$, namely
\begin{align}\label{eq:opt1}
	 \min_{\bm{\alpha}\in\mathbb{R}^T}~\frac{1}{T}\sum_{t=1}^T{\cal C}(\bm{\alpha}^{\top}\mathbf{k}(\mathbf{x}_t),y_t)+\lambda \Omega\left(\bm{\alpha}^{\top}\mathbf{K}\bm{\alpha}\right)
\end{align}
where $\mathbf{k}^\top(\mathbf{x}_t)$ is the $t$th row of the matrix $\mathbf{K}$. 
While a scalar $y_t$ is used here for brevity, coverage extends readily to vectors $\{\mathbf{y}_t\}$. 

Note that \eqref{opt0} relies on: i) a known pre-selected kernel $\kappa$; and ii) having $\{\bbx_t, y_t\}_{t=1}^T$ available in batch form. 
{A key observation here is that the dimension of the variable $\bbalpha$ in \eqref{eq:opt1} grows with time $T$ (or, the number of samples in the batch form), making it less scalable in online implementation.}
In the ensuing section, an online MKL method will be proposed to select $\kappa$ as a superposition of multiple kernels, when the data become available online.  

\section{Online MKL in static environments}\label{sec3}

In this section, we develop an online learning approach that builds on the notion of \textbf{ra}ndom features \citep{rahimi2007,felix2016}, and leverages in a unique way multi-\textbf{ker}nel approximation -- two tools justifying our acronym \textbf{Raker} used henceforth. 

\subsection{RF-based single kernel learning}\label{subsec.rf}

{To cope with the curse of dimensionality in optimizing \eqref{eq:opt1}, we will reformulate the functional optimization problem \eqref{opt0} as a parametric one with the dimension of optimization variables not growing with time. In this way, powerful toolboxes from convex optimization and online learning in vector spaces can be leveraged. 
We achieve this goal by judiciously using RFs. Although generalizations will follow, this subsection is devoted to RF-based single kernel learning, where basics of kernels, RFs, and online learning will be revisited.} 

As in \citep{rahimi2007}, we will approximate $\kappa$ in \eqref{eq:sol0} using shift-invariant kernels that satisfy $\kappa(\bbx_t,\bbx_{t'})=\kappa(\bbx_t-\bbx_{t'})$. For $\kappa(\bbx_t-\bbx_{t'})$ absolutely integrable,
its Fourier transform $\pi_{\kappa} (\bf v)$ exists and represents the power spectral density, which upon normalizing to ensure $\kappa(\mathbf{0})=1$, can also be 
viewed as a probability density function (pdf); hence,  
%
\begin{eqnarray}
\label{ieq.kx1}
\kappa(\bbx_t-\bbx_{t'}) =\int \pi_{\kappa}(\bbv)e^{j\bbv^\top(\bbx_t-\bbx_{t'})} d\bbv 
	:=\mathbb{E}_{\bbv}\big[e^{j\bbv^\top(\bbx_t-\bbx_{t'})}\big]	
\end{eqnarray}
where the last equality is just the definition of the expected value. Drawing a sufficient number of $D$ independent and identically distributed (i.i.d.) samples $\{\bbv_i\}_{i=1}^D$ from $\pi_{\kappa}(\bbv)$, the ensemble mean in \eqref{ieq.kx1} can be approximated by the sample average  
\begin{equation}\label{ieq.kx3}
\hat{\kappa}_c(\bbx_t,\bbx_{t'}):=\frac{1}{D}\sum_{i=1}^D e^{j\bbv_i^\top(\bbx_t-\bbx_{t'})}:=\bm{\zeta}_{\bV}^{\dag}(\bbx_t)\bm{\zeta}_{\bV}(\bbx_{t'})
\end{equation}
where $\bbV:=[\bbv_1, \dots, \bbv_D]^\top \in \mathbb{R}^{D\times d}$, symbol $\dag$ represents the Hermitian (conjugate-transpose) operator, and $\bm{\zeta}_{\bV}(\bbx)$ the complex RF vector 
\begin{equation}\label{ieq.kx4}
	\bm{\zeta}_{\bV}(\bbx) :=\frac{1}{\sqrt{D}}\left[e^{j\bbv_1^\top \bbx},\ldots,e^{j\bbv_D^\top \bbx}\right]^{\top}~.
\end{equation}
Taking expected values on both sides of \eqref{ieq.kx3} and using \eqref{ieq.kx1} yields
$\mathbb{E}_{\bbv}[\hat{\kappa}_c(\bbx_t,\bbx_{t'})]=\kappa(\bbx_t,\bbx_{t'})$, which means
${\hat \kappa}_c$ is unbiased. Likewise, ${\hat \kappa}_c$ can be shown consistent since  $\mathbb{V}{\rm ar} [\hat{\kappa}_c(\bbx_t,\bbx_{t'})] \propto D^{-1}$ vanishes as $D\rightarrow \infty$. Finding $\pi_{\kappa}(\bbv)$ requires $d$-dimensional Fourier transform of $\kappa$, generally through numerical integration. For a number of popular kernels however, $\pi_{\kappa}(\bbv)$ is available in closed form. Taking the Gaussian kernel as an example, where $\kappa_G(\bbx_t,\bbx_{t'})=\exp\big(\|\bbx_t-\bbx_{t'}\|_2^2/(2\sigma^2)\big)$, has Fourier transform corresponding to the pdf $\pi_G(\bbv)=\mathcal{N}(0,\sigma^{-2}\bbI)$.

Instead of the \emph{complex} RFs $\{\bm{\zeta}_{\bV}(\bbx_t)\}$ in \eqref{ieq.kx4} forming the linear kernel estimator $\hat{\kappa}_c$ in \eqref{ieq.kx3}, one can consider its real part ${\hat \kappa} (\bbx_t,\bbx_{t'}):= \Re\{\hat{\kappa}_c (\bbx_t,\bbx_{t'})\}$ that 
is also an unbiased estimator of $\kappa$. 
Defining the real RF vector $\bbz_{\bV}(\bbx):=
[\Re^\top \{\bm{\zeta}_{\bV}(\bbx_t)\}, \Im^\top \{\bm{\zeta}_{\bV}(\bbx_{t})\}]^\top$, this real kernel estimator becomes  (cf. \eqref{ieq.kx3})
\begin{align}\label{eq.ker-quad}
\hat{\kappa}(\bbx_t,\bbx_{t'})=\bbz_{\bV}^\top(\bbx_t)\bbz_{\bV}(\bbx_{t'})
\end{align}
where the $2D\times 1$ \emph{real} RF vector can be written as 
\begin{equation}\label{rep:z}
\bbz_{\bV}(\bbx) =\frac{1}{\sqrt{D}}\,\left[\sin(\bbv_1^\top \bbx),\dots, \sin(\bbv_D^\top \bbx), \cos(\bbv_1^\top \bbx), \ldots, \cos(\bbv_D^\top \bbx)\right]^{\top}\:.
\end{equation}

 Hence, the nonlinear function that is optimal in the sense of \eqref{opt0} can be approximated by a linear one in the new $2D$-dimensional RF space, namely (cf. \eqref{eq:sol0} and \eqref{eq.ker-quad})
\begin{align}
\label{eq:rf:fx}
	\hat{f}^{\rm RF}(\bbx)=\sum_{t=1}^T \alpha_t \bbz_{\bV}^\top(\bbx_t)\bbz_{\bV}(\bbx):=\bbtheta^\top\bbz_{\bV}(\bbx)
\end{align}
where $\bbtheta^{\top}:=\sum_{\tau=1}^T \alpha_{\tau} \bbz_{\bV}^{\top}(\bbx_{\tau})$ is the new weight vector of size $2D$ whose dimension does not increase with number of data samples $T$. 

While the solution $\hat{f}$ in \eqref{eq:sol0} is the superposition of nonlinear functions $\kappa$, its RF approximant $\hat{f}^{\rm RF}$ in \eqref{eq:rf:fx} is a linear function of $\bbz_{\bV}(\bbx)$. As a result, the loss becomes 
 \begin{align}\label{eq:loss-fun}
 	{\cal L}_t\big(f(\bbx_t)\big):={\cal C}(f(\bbx_t),y_t)+\lambda\Omega\left(\|f\|_{\mathcal{H}}^2\right)={\cal C}\big(\bbtheta^\top\bbz_{\bV}(\bbx_t),y_t\big)+\lambda \Omega\left(\|\bbtheta\|^2\right)
 \end{align}
where $\|\bbtheta\|^2:=\sum_{t}\sum_{t'}\alpha_t\alpha_{t'}\bbz_{\bbV}^\top(\bbx_t)\bbz_{\bbV}(\bbx_{t'}):=\|f\|_{\cal H}^2$; and the online learning task is 
\begin{equation}\label{eq:rf-task}
\min_{\bm{\theta}\in\mathbb{R}^{2D}}\, \sum_{t=1}^T{\cal L}\left(\bbtheta^\top\bbz_{\bV}(\bbx_t),y_t\right)\!,~{\rm with}~{\cal L}\big(\bbtheta^\top\bbz_{\bV}(\bbx_t),y_t\big):={\cal C}\big(\bbtheta^\top\bbz_{\bV}(\bbx_t),y_t\big)+\lambda \Omega\big(\|\bbtheta\|^2\big).	
\end{equation}
{Compared with the functional optimization in \eqref{opt0}, the reformulated problem \eqref{eq:rf-task} is parametric, and more importantly it involves only optimization variables of fixed size $2D$.
We can thus solve \eqref{eq:rf-task} using the online gradient descent iteration}, e.g., \citep{hazan2016}. Acquiring $\bbx_t$ per slot $t$, its RF $\bbz_{\bV}(\bbx_t)$ is formed as in \eqref{rep:z}, and $\bbtheta_{t+1}$ is updated online as 
  \begin{align}\label{eq:weit-rf}
  	\bbtheta_{t+1}=\bbtheta_t-\eta_t \nabla{\cal L}(\bbtheta_t^\top\bbz_{\bV}(\bbx_t),y_t)
  \end{align}
  where $\{\eta_t\}$ is the sequence of stepsizes that can tune learning rates, and $\nabla{\cal L}(\bbtheta_t^\top\bbz_{\bV}(\bbx_t),y_t)$ the gradient at $\bbtheta=\bbtheta_t$. 
Iteration \eqref{eq:weit-rf} provides \emph{a functional update} since $\hat{f}^{\rm RF}_t(\bbx)=\bbtheta_t^\top\bbz_{\bV}(\bbx)$, but the upshot of involving RFs is that this approximant is in the span of $\{\bbz_{\bV}(\bbx), \forall \bbx\in {\cal X}\}$. 
Since $\mathbb{E}[\hat{\kappa}]=\kappa$, we find readily that $\mathbb{E}[\hat{f}^{\rm RF}]=\hat{f}$; in words, unbiasedness of the kernel approximation ensures that the RF-based function approximant is also unbiased.
  
\paragraph{Variance-reduced RF.} {\color{black}Besides unbiasedness,  performance of the RF approximation is also influenced by the variance of RFs. Note that} the variance of ${\hat \kappa}$ in \eqref{eq.ker-quad} is of order ${\cal O} (D^{-1})$, but its scale can be reduced if $\bbV$ is formed to have orthogonal rows~\citep{felix2016}. Specifically for 
a Gaussian kernel with bandwidth $\sigma^2$, recall that $\bbV={\sigma}^{-1}\bbG$ in \eqref{rep:z}, where each entry of $\bbG$ is drawn from ${\cal N} (0,1)$. 
For the variance-reduced orthogonal (O)RF with $D=d$, one starts with Q-R factorization of $\bbV = \bbQ \bbR$, and uses the $d \times d$ factor $\bbQ$ along with a diagonal matrix 
$\bbLambda$, to form~\citep{felix2016}
 \begin{equation}
 \label{eq:orf}
 \bV_{\rm ORF}= \sigma^{-1} \: \bbLambda\bbQ
 \end{equation}
where the diagonal entries of $\bbLambda$ are drawn i.i.d. from the $\chi$ distribution with $d$ degrees of freedom, to ensure unbiasedness of the kernel approximant.
For $D>d$, one selects $D=\nu d$ with $\nu >1$ integer, and generates independently $\nu$ matrices each of size $d \times d$ as in \eqref{eq:orf}. The final ${\bf V}_{\rm ORF}$
is formed by concatenating these $d\times d$ sub-matrices.
The upshot of ORF is that \citep{felix2016}
$\mathbb{V}{\rm ar}(\hat{\kappa}_{\rm ORF}(\bbx_t,\bbx_{t'}))\leq\mathbb{V}{\rm ar}(\hat{\kappa}(\bbx_t,\bbx_{t'}))$. As we have also confirmed via simulated tests, ORF-based function approximation can attain a prescribed accuracy with considerably less ORFs than what required by its RF-based counterpart.  

{The RF-based online single kernel learning scheme in this section presumes that $\kappa$ is known a priori.} Since this is not generally possible, 
it is prudent to adaptively select kernels by superimposing multiple kernel functions from a prescribed dictionary. This superposition will play a key role in {the RF-based online MKL approach presented next.}  

\subsection{Raker for online MKL}\label{sec.rf-mkl}
Specifying the kernel that ``shapes'' ${\cal H}$ is a critical choice for single kernel learning, since different kernels yield function estimates of variable accuracy. To deal with this, combinations of kernels from a prescribed and sufficiently rich dictionary $\{\kappa_p\}_{p=1}^P$ can be employed in \eqref{opt0}. Each combination belongs to the convex hull $\bar{{\cal K}}:=\{\bar{\kappa}=\sum_{p=1}^P \bar{\alpha}_p \kappa_p,\, \bar{\alpha}_p\geq 0,\,\sum_{p=1}^P\bar{\alpha}_p=1\}$, and is itself a kernel \citep{scholkopf2002}. With $\bar{\cal H}$ denoting the RKHS induced by $\bar{\kappa}\in \bar{{\cal K}}$, one then solves \eqref{opt0} with ${\cal H}$ replaced by 
$\bar{\cal H}:={\cal H}_1\bigoplus\cdots\bigoplus{\cal H}_P$,
where $\{{\cal H}_p\}_{p=1}^P$ represent the RKHSs corresponding to $\{\kappa_p\}_{p=1}^P$~\citep{micchelli2005}. 

The candidate function ${\bar f} \in \bar{\cal H}$ is expressible in a separable form as $\bar{f}(\bbx): =\sum_{p=1}^P {\bar f}_p(\bbx)$, where ${\bar f}_p(\bbx)$ belongs to $\mathcal{H}_p$, for $p\in{\cal P}:=\{1, \ldots, P\}$. To add flexibility per kernel in our ensuing online MKL scheme, we let wlog $\{{\bar f}_p = {w}_p f_p\}_{p=1}^P$, and seek functions of the form 
\begin{align}\label{eq:fp}
f(\bbx): =\sum_{p=1}^P \bar{w}_p f_p(\bbx)\in\bar{\cal H}
\end{align}
where $f:={\bar f}/\sum_{p=1}^P w_p$, and the normalized weights $\{\bar{w}_p:=w_p/\sum_{p=1}^P w_p\}_{p=1}^P$ satisfy $\bar{w}_p\geq 0$, and $\sum_{p=1}^P\bar{w}_p=1$. Plugging \eqref{eq:fp} into \eqref{opt0},
MKL solves the nonconvex problem
\begin{subequations}\label{opt1}
	 \begin{align}
	 \min_{\{\bar{w}_p\}, \{f_p\}}~&\frac{1}{T}\sum_{t=1}^T{\cal C}\left(\sum_{p=1}^P \bar{w}_p f_p(\bbx_t),y_t\right)+\lambda\Omega\left(\left\|\sum_{p=1}^P \bar{w}_p f_p\right\|_{\mathcal{\bar H}}^2\right)\label{eq.opt1a}\\
	 {\rm s.~to}&~~\sum_{p=1}^P\bar{w}_p=1,~\bar{w}_p\geq 0,~p\in{\cal P}\label{eq.opt1b}\\
	&~~ f_p\in \mathcal{H}_p,~p\in{\cal P}. \label{eq.opt1c}
\end{align}
\end{subequations}
If $\Omega$ is convex over $f$, then \eqref{eq.opt1a} is biconvex, meaning it is convex wrt $\{f_p\}$ ($\{\bar{w}_p\}$) when  $\{\bar{w}_p\}$ ($\{f_p\}$) is given. Leveraging biconvexity, existing batch MKL schemes solve \eqref{opt1} via alternating minimization that is known not to scale well with $P$ and $T$~\citep{micchelli2005,cortes2009,gonen2011}. 

To deal with scalability,  our novel approach will leverage for the first time (O)RFs in a uniquely principled MKL formulation to end up with an efficient online learning approach. To this end, we will minimize a cost that upper bounds that in \eqref{eq.opt1a}, namely 
\begin{align}\label{opt2}
	 \min_{\{\bar{w}_p\}, \{f_p\}}~&\frac{1}{T}\sum_{t=1}^T\sum_{p=1}^P \bar{w}_p\, {\cal C}\left(f_p(\bbx_t),y_t\right)+\lambda\sum_{p=1}^P \bar{w}_p\,\Omega\left(\left\| f_p\right\|_{\mathcal{H}_p}^2\right)~~~{\rm s.~to}~~\eqref{eq.opt1b}~{\rm and}~\eqref{eq.opt1c}
\end{align}
where Jensen's inequality confirms that under \eqref{eq.opt1b} the cost in \eqref{opt2} upper bounds that of \eqref{eq.opt1a}. A key advantage of \eqref{opt2} is that its objective is separable across kernel `atoms.' 

We will exploit this separability jointly with the 
RF-based function approximation per kernel, to formulate our scalable online MKL task as
\begin{align}\label{eq:raker-task}
\!\!\min_{\{\bar{w}_p\}, \{\hat{f}_p^{\rm RF}\}} \sum_{t=1}^T\sum_{p=1}^P \bar{w}_p\, {\cal L}_t\left(\hat{f}_p^{\rm RF}(\bbx_t)\right)~~{\rm s.~to}~~\eqref{eq.opt1b}~~{\rm and}~~\hat{f}_p^{\rm RF}\!\in\!\left\{\hat{f}_p(\bbx)\!=\!\bbtheta^{\top}\bbz_{\bbV_p}(\bbx)\right\}	
\end{align}
{where we interchangeably use ${\cal L}_t(\hat{f}(\bbx_t))$ as defined in \eqref{eq:loss-fun} and ${\cal L}\big(\bbtheta^\top\bbz_{\bV}(\bbx_t),y_t\big)$ as in \eqref{eq:rf-task}.}
We will efficiently solve \eqref{eq:raker-task} `on-the-fly' using our Raker algorithm, and what more, we will provide analytical performance guarantees. Our iterative solution
will update separately each $\hat{f}_p^{\rm RF}$ as in Section \ref{subsec.rf} using the scalable (O)RF-based function approximation scheme. Given $\bbx_t$, an RF vector $\bbz_p(\bbx_t)$ will be generated per $p$ from pdf $\pi_{\kappa_p}(\bbv)$  (cf. \eqref{rep:z}), where we let $\bbz_p(\bbx_t):=\bbz_{\bbV_p}(\bbx_t)$ for notational brevity. Hence, for each $p$ and slot $t$, we have 
  \begin{align}\label{eq.klp-output}
  	\hat{f}_{p,t}^{\rm RF}(\bbx_t)=\bbtheta_{p,t}^\top \bbz_p(\bbx_t)
  \end{align}
and as in \eqref{eq:weit-rf}, $\bbtheta_{p,t}$ is updated via
  \begin{align}\label{eq.klp-weight}
  	\bbtheta_{p,t+1}=\bbtheta_{p,t}-\eta \nabla{\cal L}(\bbtheta_{p,t}^\top\bbz_p(\bbx_t),y_t).
  \end{align}

As far as solving for $\bar{w}_{p,t}$, since it resides on a probability simplex \eqref{eq.opt1b}, our idea is to employ a multiplicative update (a.k.a. exponentiated gradient descent), e.g.,~\citep{hazan2016}. 
Specifically, the un-normalized weights are found first as
\begin{align}\label{eq.mkl-weight0}
w_{p,t+1}=\arg\min_{w_p}\, \eta \: {\cal L}_t\left(\hat{f}_{p,t}^{\rm RF}(\bx_t)\right)(w_p-w_{p,t})+{\cal D}_{\rm KL}(w_p\|w_{p,t})
\end{align}
where ${\cal D}_{\rm KL}(w_p\|w_{p,t}):=w_p\log (w_p/w_{p,t})$ is the KL-divergence.
It can be readily verified that \eqref{eq.mkl-weight0} admits the following closed-form update
\begin{align}\label{eq.mkl-weight}
w_{p,t+1}=w_{p,t}\exp\left(-\eta{\cal L}_t\left(\hat{f}_{p,t}^{\rm RF}(\bx_t)\right)\right)
\end{align}
where $\eta\in(0,1)$ is a chosen constant that controls the adaptation rate of $\{w_{p,t}\}$. Having found $\{w_{p,t}\}$ as in \eqref{eq.mkl-weight}, the normalized weights in \eqref{eq:fp} are obtained as $\bar{w}_{p,t}:=w_{p,t}/\sum_{p=1}^P w_{p,t}$. 
Update \eqref{eq.mkl-weight} is intuitively pleasing because when $\hat{f}_{p,t}^{\rm RF}$ contributes a larger loss relative to other $\hat{f}_{p',t}^{\rm RF}$ with $p' \neq p$ at slot $t$, the corresponding $w_{p,t+1}$ decreases more than the other weights in the next time slot. In other words, a more accurate RF-based approximant tends to play more important role in predicting the upcoming data.

\begin{algorithm}[t] 
	\caption{Raker for online MKL in static environments}\label{algo:omkl:rf}
	\begin{algorithmic}[1]
		\State\textbf{Input:}~Kernels $\kappa_p, ~p=1,\ldots, P$, step size $\eta>0$, and number of random features $D$.
		
		
		\State\textbf{Initialization:}~$\bbtheta_1=\mathbf{0}$.
		\For {$t = 1, 2,\ldots, T$}

		\State Receive a streaming datum $\bbx_t$. 
		\State Construct $\bbz_{p}(\bbx_t)$ via \eqref{rep:z} using $\kappa_p$ for $p=1,\dots, P$. 
		\State Predict $\hat{f}_t^{\rm RF}(\bbx_t):=\sum_{p=1}^P \bar{w}_{p,t} \hat{f}_{p,t}^{\rm RF}(\bbx_t)$ with $\hat{f}_{p,t}^{\rm RF}(\bbx_t)$ in \eqref{eq.klp-output}.  
		\State Observe loss function ${\cal L}_t$, incur ${\cal L}_t(\hat{f}_t^{\rm RF}(\bbx_t))$.
		\hspace{1.cm} \For {$p=1, \ldots, P$}
		\State Obtain loss ${\cal L}(\bbtheta_{p,t}^\top\bbz_p(\bbx_t),y_t)$ or ${\cal L}_t(\hat{f}_{p,t}^{\rm RF}(\bbx_t))$.
		\State Update $\bbtheta_{p,t+1}$ via \eqref{eq.klp-weight}.
		\State Update $w_{p,t+1}$  via \eqref{eq.mkl-weight}.
		\EndFor
		\EndFor
	\end{algorithmic}
\end{algorithm}

\vspace*{0.1cm}
\noindent
{\bf Remark 1}. 
{The update \eqref{eq.mkl-weight} resembles the online learning paradigm, a.k.a. \emph{online prediction with} (weighted) \emph{expert advices}~\citep{vovk1995,cesa2006}. 
Building on but going beyond OMKL in \citep{sahoo2014online}, the idea here is to view MKL with \emph{RF-based function approximants} as a weighted combination of advices from an ensemble of $P$ function approximants (experts). Besides permeating benefits from online learning to MKL, what is distinct here relative to~\citep{vovk1995,cesa2006} is that each function approximant also performs online learning for self improvement (cf. \eqref{eq.klp-weight}).} 

In summary, our Raker for static (or slow-varying) dynamics is listed as Algorithm \ref{algo:omkl:rf}. 

\paragraph{Memory requirement and computational complexity.} 
At the $t$-th iteration, our Raker in Algorithm \ref{algo:omkl:rf} needs to store a real $2D$ RF vector, and its corresponding weight vector per $\kappa_p$. Hence, the memory required is of order $\mathcal{O}(dDP)$. Regarding computational overhead, the per-iteration complexity (e.g., calculating inner products) is again of order $\mathcal{O}(dDP)$. Compared with the complexity of $\mathcal{O}(tdP)$ for OMKL by~\citep{sahoo2014online}, or, $\mathcal{O}(t^3P)$ when matrix inversion required for the batch MKL, e.g.,  \citep{bazerque2013nonparametric}, the Raker is clearly more scalable, as $t$ grows. Even when OMKL is confined to a budget of $B$ past samples, the corresponding complexity of $\mathcal{O}(dBP)$ is comparable to that of Raker. 
This speaks for Raker's merits, whose performance guarantees will be proved analytically, and also demonstrated by numerical tests to outperform budgeted schemes.

\paragraph{Application examples: Online MKL regression and classification.}
To appreciate the usefulness of RF-based online MKL, consider first nonlinear regression, where given samples $\{\bbx_t\in\mathbb{R}^d, y_t\in\mathbb{R}\}_{t=1}^T$, the goal is to find a nonlinear function $f\in{\cal H}$, such that $y_t=f(\bbx_t)+e_t$.  The criterion is to minimize the regularized prediction error of $y_t$, typically using the LS loss ${\cal L}(f(\bbx_t),y_t):=[y_t - f(\bbx_t)]^2+\lambda\|f\|_{\cal H}^2$, whose gradient is (cf. \eqref{eq.klp-weight})
\begin{align}
 \nabla{\cal L}\left(\bbtheta_{p,t}^\top\bbz_p(\bbx_t),y_t\right)=2(\bbtheta_{p,t}^\top\bbz_p(\bbx_t)-y_t)\bbz_p(\bbx_t)+2\lambda \bbtheta_{p,t}.
\end{align}
It is clear that the per iteration complexity of Raker is only related to the dimension of $\bbz_p(\bbx_t)$, and does not increase over time. 

For nonlinear classification, consider kernel-based perceptron and kernel-based logistic regression, which aim at learning a nonlinear classifier that best approximates either $y_t$ or the pdf of $y_t$ conditioned on $\bbx_t$. 
With binary labels $\{\pm 1\}$, the perceptron solves \eqref{opt0} with ${\cal L}(f(\bbx_t),y_t)=\max(0,1-y_t f(\bbx_t))+\lambda\|f\|_{\cal H}^2$, which equals zero if $y_t=f(\bbx_t)$, otherwise it equals $1$. Raker's gradient in this case is (cf. \eqref{eq.klp-weight})
\begin{align}
	\nabla {\cal L}\left(\bbtheta_{p,t}^\top\bbz_p(\bbx_t),y_t\right)=-2y_t{\cal C}(\bbtheta_{p,t}^\top\bbz_p(\bbx_t),y_t)\bbz_p(\bbx_t)+2\lambda \bbtheta_{p,t}.
\end{align}
Accordingly, given ${\bf x}_t$, logistic regression postulates that ${\rm Pr}(y_t=1|\bbx_t )=1/(1+\exp(f(\bbx_t)))$.
Here the gradient of Raker takes the form (cf. \eqref{eq.klp-weight})
\begin{align}\label{eq.grad-log}
	\nabla {\cal L}\left(\bbtheta_{p,t}^\top\bbz_p(\bbx_t),y_t\right)=\frac{2y_t\exp(y_t\bbtheta_{p,t}^\top\bbz_p(\bbx_t))}{1+\exp(y_t\bbtheta_{p,t}^\top\bbz_p(\bbx_t))}\bbz_p(\bbx_t)+2\lambda \bbtheta_{p,t}.
\end{align}
{To compare alternatives on equal footing, the numerical tests in Section \ref{sec:test} will deal with kernel-based regression and classification.}

\subsection{Static regret analysis of Raker}\label{sec.stat-reg}
 
To analyze the performance of Raker, we assume that the following conditions are satisfied.

\vspace{0.1cm}
\noindent\textbf{(as1)}
\emph{Per slot $t$, the loss function ${\cal L}(\bbtheta^\top\bbz_{\bV}(\bbx_t),y_t)$ in \eqref{eq:rf-task} is convex w.r.t. $\bbtheta$.}

\vspace{0.1cm}
\noindent\textbf{(as2)}
\emph{For $\bm{\theta}$ belonging to a bounded set ${\bbTheta}$ with $\|\bbtheta\|\leq C_{\theta}$, the loss is bounded; that is, ${\cal L}(\bbtheta^\top\bbz_{\bV}(\bbx_t),y_t)\in[-1,1]$, and has bounded gradient, meaning, $\|\nabla {\cal L}(\bbtheta^\top\bbz_{\bV}(\bbx_t),y_t)\|\leq L$.}

\vspace{0.1cm}
\noindent\textbf{(as3)}
\emph{Kernels $\{\kappa_p\}_{p=1}^P$ are shift-invariant, standardized, and bounded, that is, $\kappa_p(\mathbf{x}_i,\mathbf{x}_j)\!\leq\! 1,\,\forall \mathbf{x}_i,\mathbf{x}_j$; and w.l.o.g. they also have bounded entries, meaning $\|\mathbf{x}\|\leq 1$.}  
\vspace{0.1cm}

Convexity of the loss under (as1) is satisfied by the popular loss functions including the square loss and the hinge loss. As far as (as2), it ensures that the losses, and their gradients are bounded, meaning they are $L$-Lipschitz continuous. 
While boundedness of the losses commonly holds since $\|\bbtheta\|$ is bounded, Lipschitz continuity is also not restrictive. Considering kernel-based regression as an example, the gradient is $(\bm{\theta}^{\top} \mathbf{z}_{\bV}(\mathbf{x}_t)-y_t) \mathbf{z}_{\bV}(\mathbf{x}_t)+\lambda\bbtheta$. Since the loss is bounded, e.g., $\|\bm{\theta}^{\top} \mathbf{z}_{\bV}(\mathbf{x}_t)-y_t\| \leq 1$, and the RF vector in \eqref{rep:z} can be bounded as $\|\mathbf{z}_{\bV}(\mathbf{x}_t)\|\leq 1$, the constant is $L:= 1+\lambda C_{\theta}$ using the Cauchy-Schwartz inequality. 
Kernels satisfying conditions in (as3) include Gaussian, Laplacian, and Cauchy \citep{rahimi2007}.   
In general, (as1)-(as3) are standard in online convex optimization (OCO) \citep{shalev2011,hazan2016}, and in kernel-based learning \citep{micchelli2005,rahimi2007,lu2016large}.

With regard to the performance of an online algorithm, static regret is commonly adopted as a metric by most OCO schemes to measure the difference
between the aggregate loss of an OCO algorithm, and that of the best
fixed function approximant in hindsight, e.g.,~\citep{shalev2011,hazan2016}. Specifically, for a generic sequence $\{\hat{f}_t\}$ generated by an RF-based kernel learning algorithm ${\cal A}$, its static regret is
\begin{align}\label{eq.sta-reg}
    {\rm Reg}_{\cal A}^{\rm s}(T):=\sum_{t=1}^T {\cal L}_t(\hat{f}_t(\bbx_t))-\sum_{t=1}^T{\cal L}_t(f^*(\bbx_t))
\end{align}
where $\hat{f}_t$ will henceforth represent $\hat{f}_t^{\rm RF}$ without the superscript for notational brevity; and, $f^*(\cdot)$ is obtained as the batch solution
\begin{equation}\label{eq.slot-opt}
 f^*(\cdot)   \in\arg\min_{\{f_p^*,\,p\in{\cal P}\}}\,\sum_{t=1}^T {\cal L}_t(f_p^*(\bbx_t))~~~{\rm with}~~~f_p^*(\cdot)\in\arg\min_{f\in{\cal F}_p} \,\sum_{t=1}^T {\cal L}_t(f(\bbx_t))
\end{equation}
with ${\cal F}_p:={\cal H}_p$, and ${\cal H}_p$ representing the RKHS induced by $\kappa_p$. 
Using \eqref{eq.sta-reg} and \eqref{eq.slot-opt}, we first establish the static regret of our Raker approach in the following lemma.
 
 \begin{lemma}
\label{lemma4}
	Under (as1), (as2), and with $\hat{f}_p^*$ as in \eqref{eq.slot-opt} with ${\cal F}_p:=\{\hat{f}_p|\hat{f}_p(\bbx)=\bbtheta^{\top}\mathbf{z}_p(\bbx),\,\forall \bbtheta\in\mathbb{R}^{2D}\}$, the sequences $\{\hat{f}_{p,t}\}$ and $\{\bar{w}_{p,t}\}$ generated by Raker satisfy the following bound
\begin{align}
	\label{eq.mkl.sreg}
\sum_{t=1}^T{\cal L}_t\bigg(\sum_{p=1}^P \bar{w}_{p,t} \hat{f}_{p,t}(\bbx_t)\bigg)-\sum_{t=1}^T{\cal L}_t\left(\hat{f}_p^*(\bbx_t)\right)
\leq \frac{\ln P}{\eta}+\frac{\|\bbtheta_p^*\|^2}{2\eta}+\frac{\eta L^2T}{2}+\eta T
	\end{align}
where $\bbtheta_p^*$ is associated with the best RF function approximant $\hat{f}_p^*(\bbx)=\left(\bbtheta_p^*\right)^{\top}\mathbf{z}_p(\bbx)$.
\end{lemma}

\begin{proof}
	See Appendix \ref{app.pf.lemma4}.
\end{proof}

Besides Raker's static regret bound, the next theorem compares the Raker loss 
relative to that of the best functional estimator in the original RKHS.
\begin{customthm}{1}\label{theorem0}
Under (as1)-(as3) and with $f_p^*$ in \eqref{eq.slot-opt} belonging to the RKHS ${\cal H}_p$, for a fixed $\epsilon>0$, the following bound holds with probability at least $1-2^8\big(\frac{\sigma_p}{\epsilon}\big)^2 \exp \big(\frac{-D\epsilon^2}{4d+8}\big)$
\begin{align}\label{eq.sreg.f}
	\sum_{t=1}^T{\cal L}_t\left(\sum_{p=1}^P \bar{w}_{p,t} \hat{f}_{p,t}(\bbx_t)\right)-&\min_{p\in\{1,\ldots,P\}}\sum_{t=1}^T{\cal L}_t\left(f_p^*(\bbx_t)\right)\nonumber\\
	\leq &\frac{\ln P}{\eta}+\frac{(1+\epsilon)C^2}{2\eta}\!+\!\frac{\eta L^2T}{2}+\eta T\!+\!\epsilon LTC
\end{align}
where $C$ is a constant, while $\sigma_p^2:=\mathbb{E}_{\bV}^{\pi_{\kappa_p}}[\|\bbv\|^2]$ is the second-order moment of the RF vector norm. Setting $\eta=\epsilon={\cal O}(1/\sqrt{T})$ in \eqref{eq.sreg.f}, the static regret in \eqref{eq.sta-reg} leads to
\begin{align}
\label{eq:sreg:11}
	 {\rm Reg}_{\rm Raker}^{\rm s}(T)= {\cal O}(\sqrt{T}).
\end{align}
\end{customthm}

\begin{proof}
	See Appendix \ref{app.pf.theorem0}.
\end{proof}

Observe that the probability of \eqref{eq.sreg.f} to hold grows as $D$ increases, and one can always find a $D$ to ensure a positive probability for a given $\epsilon$. Bearing this in mind, we will henceforth use ``with high probability'' (w.h.p.) to summarize the sense \eqref{eq.sreg.f} and \eqref{eq:sreg:11} hold.  Theorem \ref{theorem0} establishes that with proper choice of parameters, the Raker achieves sub-linear regret relative to the best static function approximant in \eqref{eq.slot-opt}.

\section{Online MKL in Environments with Unknown Dynamics}\label{sec4}

Our Raker in Section \ref{sec3} combines an ensemble of kernel learners `on the fly,' and performs on average as the ``best'' fixed function, thus fulfilling the learning objective in environments with zero (or slow) dynamics. To broaden its scope to environments with unknown dynamics, this section introduces an \textbf{ada}ptive \textbf{Raker} approach (termed \textbf{AdaRaker}).

\subsection{AdaRaker with hierarchical ensembles}

As with any online learning algorithm, the choice of $\eta$ in \eqref{eq.klp-weight} and \eqref{eq.mkl-weight} affects the performance critically. 
Especially in environments with unknown dynamics, a large $\eta$ improves the tracking ability of fast-varying functions, while a smaller one allows improved estimation of slow-varying parameters $\{\boldsymbol{\theta}_t,w_{p,t}\}$.  
The optimal choice of $\eta_t$ clearly depends on the variability of the optimal function estimator \citep{kivinen2004,besbes2015}. 
Selecting $\{\eta_t\}$ however, is formidably challenging if the environment dynamics are unknown.

\begin{figure}[t]
\vspace{-.2in}
\hspace{-.1in}
\includegraphics[width=1\textwidth]{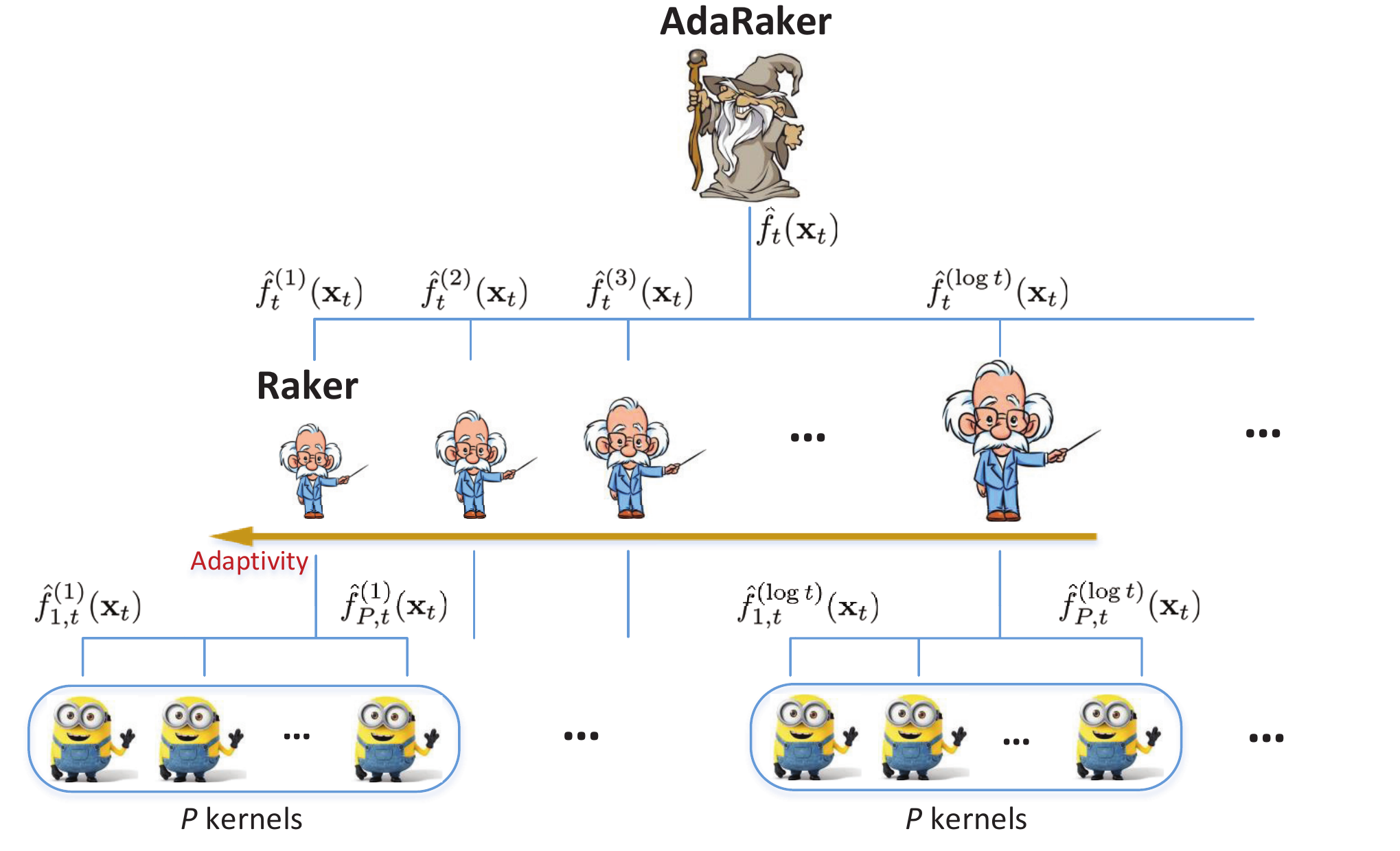}
\caption{Hierarchical AdaRaker structure. Experienced experts in the middle layer present a Raker instance, where the size of expert cartoons is proportional to the  interval length.}\label{fig:diag}
\vspace{-.1in}
\end{figure}

Toward addressing this challenge, our idea here is to hedge between multiple Raker learners with different learning rates. Specifically, we view each Raker instance in Algorithm \ref{algo:omkl:rf} as a black box algorithm $\mathcal{A}_I$, where the subscript $I$ represents the algorithm running on interval $I:=[\underline{I},\bar{I}]$ starting from slot $\underline{I}$ to slot $\bar{I}$. 
Let a pre-selected set ${\cal I}$ collect all these intervals, the design of which will be specified later. 
At the beginning of each interval $I\in {\cal I}$, a new instance of the online Raker algorithm $\mathcal{A}_I$ is initialized with an interval-specific learning rate $\eta^{(I)}:=\min\{1/2,\eta_0/\sqrt{|I|}\}$ with constant $\eta_0>0$. 
Allowing for overlap between intervals, multiple Raker instances $\{\mathcal{A}_I\}$ will be run in parallel. Consider now collecting all active intervals at the current slot $t$ in the set 
\begin{equation}\label{eq.act-set}
	{\cal I}(t):=\{I\in {\cal I}\,|\,t\in [\underline{I},\bar{I}]\},~\forall t\in{\cal T}.
\end{equation}

For each Raker instance ${\cal A}_I$ with $I\in{\cal I}(t)$, let $\hat{f}_t^{(I)}(\cdot)$ denote its output at time $t$ that combines multiple kernel-based function estimators, and ${\cal L}_t(\hat{f}_t^{(I)}(\bbx_t))$ represent the associated instantaneous loss. 
The output of the ensemble learner ${\cal A}$ at time $t$ is the weighted combination of outputs from all learners, namely $\{\hat{f}_t^{(I)},\,\forall I\in{\cal I}(t)\}$. 
With $h_t^{(I)}$ denoting the weight of the Raker instance ${\cal A}_I$, we will update it online via 
\begin{equation}\label{eq.weight-up}
	h_{t+1}^{(I)}=
    \left\{
    \begin{array}{cl}
         {0,}~  &\text{if $t\notin I$}\\
         {\eta^{(I)},}~ &\text{if $t=\underline{I}$}\\
         h_{t}^{(I)}\exp\big(\!-\eta^{(I)}r_t^{(I)}\big),~ &\text{else}
    \end{array}
   \right.
\end{equation}
where $\underline{I}$ is the first time slot of interval $I$, and the loss of ${\cal A}_I$ \emph{relative} to the overall loss is
\begin{equation}\label{eq.rel-loss}
r_t^{(I)}={\cal L}_t(\hat{f}_t(\bbx_t))-{\cal L}_t(\hat{f}_t^{(I)}(\bbx_t)),~\forall I\in{\cal I}(t).
\end{equation}
%
Intuitively thinking, one would wish to decrease (increase) the weights of those instances with small (large) losses in future rounds. 
Using update \eqref{eq.weight-up}, and defining the normalized weight as $\bar{h}_t^{(I)}:=h_t^{(I)}/\sum_{J\in{\cal I}(t)}h_t^{(J)}$, the overall output is given by
\begin{align}\label{eq.comb-up}
	\hat{f}_t(\bbx):=\sum_{I\in{\cal I}}\bar{h}_t^{(I)}\hat{f}_t^{(I)}(\bbx)~~~{\rm with}~~~\hat{f}_t^{(I)}(\bbx):=\sum_{p\in{\cal P}}\bar{w}_{p,t}^{(I)}\hat{f}_{p,t}^{(I)}(\bbx)
\end{align}
where $\{\bar{w}_{p,t}^{(I)}\}$ are the kernel combination weights generated by Raker ${\cal A}_I$ (cf. \eqref{eq.mkl-weight}).

The AdaRaker scheme is summarized in Algorithm \ref{algo:sadapt}, and depicted in Figure \ref{fig:diag}.

\begin{algorithm}[t] 
	\caption{AdaRaker for online MKL in dynamic environments}\label{algo:sadapt}
	\begin{algorithmic}[1]
		
		\State\textbf{Initialization:} learner weights $\{h_1^{(I)}\}$, and their learning rates $\{\eta^{(I)}\}$.
		\For {$t = 1, 2,\ldots, T$}

		\State Obtain $\hat{f}_t^{(I)}(\bbx_t)$ from each Raker instance $\mathcal{A}_I$, $I\in {\cal I}(t)$.
		\vspace{1mm}
		\State Predict $\hat{f}_t(\bbx_t)$ via a weighted combination \eqref{eq.comb-up}.
		\State Observe loss function ${\cal L}_t$, and incur ${\cal L}_t(\hat{f}_t(\bbx_t))$.
		\For {$ I\in {\cal I}(t)$}
		\State Incur loss ${\cal L}_t(\hat{f}_t^{(I)}(\bbx_t))$. 
		\State Update $\hat{f}_t^{(I)}$ via Raker in Algorithm \ref{algo:omkl:rf}.
		\State Update weights $h_{t+1}^{(I)}$ via \eqref{eq.weight-up}.
		\EndFor
		\EndFor
	\end{algorithmic}
\end{algorithm}

Selecting judiciously variable-length intervals in ${\cal I}$ can affect performance critically. Such a selection criterion for achieving interval regret has been reported 
in~\citep{daniely2015}. Instead, our pursuit is a hierarchical ensemble design for online MKL in environments with unknown dynamics using scalable RF-based function approximants. 
This hierarchical design is well motivated because with long intervals, the Raker loss per interval is relatively low in slow-varying settings, but higher as the dynamics become more pronounced. On the other hand, a short interval can hedge against a possibly rapid change, but its performance on each interval could suffer if the objective stays nearly static. Bearing these tradeoffs in mind, we present next a simple yet efficient interval partitioning scheme.

\vspace{0.1cm}
\noindent\textbf{Illustration of interval sets:} 
\textit{Consider partitioning the entire horizon into intervals of length $2^0, 2^1, 2^2, \ldots$. 
Intervals of length $2^j$ with a given $j\in\mathbb{N}$ are consecutively assigned without overlap starting from $t=2^j$. In the non-overlapping case, define a set of intervals $\mathcal{I}_j=[\underline{I}_{j}, \bar{I}_{j}]$ such that each interval's length is $|\mathcal{I}_j|=\bar{I}_{j}-\underline{I}_{j}+1=2^j,\, j\in\mathbb{N}$. 
For this selection of intervals, each time slot $t$ is covered by a set of  at most $\lceil\log_2t\rceil$ intervals, which forms the active set of intervals ${\cal I}(t)$ at time $t$. See the diagram in Fig. \ref{fig:meta}.}
\vspace{0.1cm}

\begin{figure}[t]
\hspace{-.1in}
\includegraphics[width=1\textwidth]{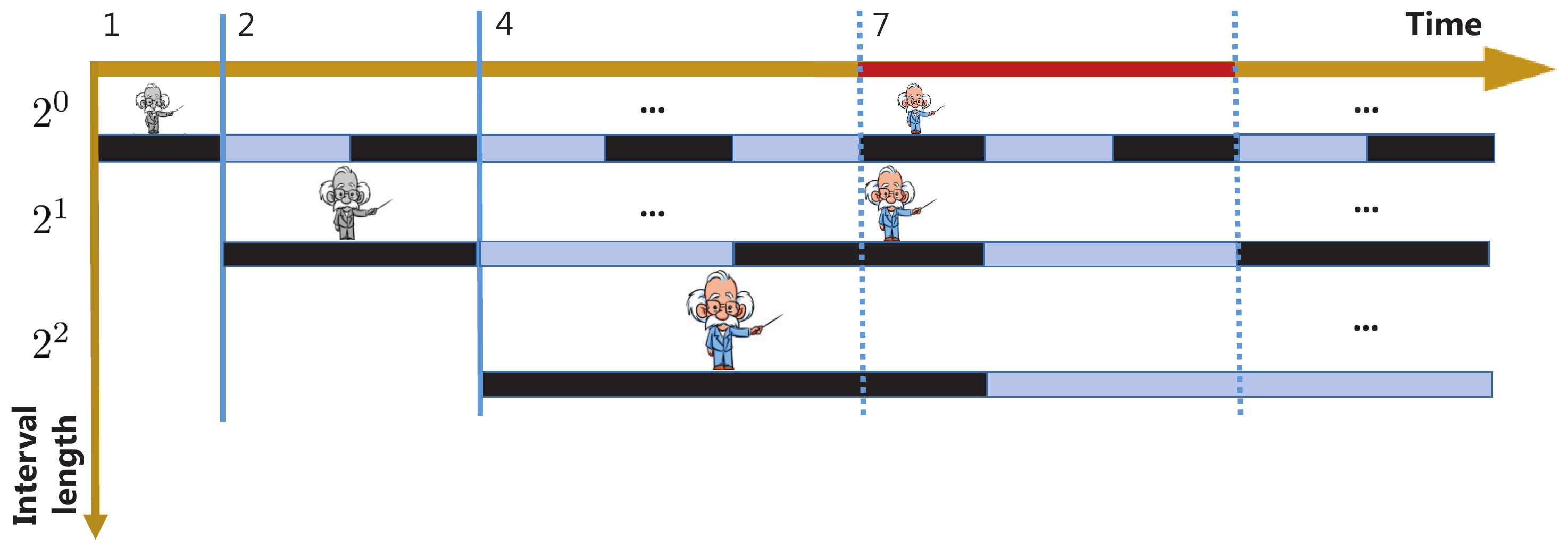}
\vspace{-.1in}
\caption{AdaRaker as an ensemble of Rakers with different learning rates: Each light/dark black interval initiates a Raker learner. At slot $7$, colored experts are active, and  gray ones are inactive.}\label{fig:meta}
\end{figure}

\subsection{Dynamic regret analysis of AdaRaker}

The static regret in Theorem 1 is with respect to a time-invariant optimal function estimator benchmark. In dynamic environments however, this optimal function benchmark may change over time - what justifies this subsection's performance analysis of AdaRaker. 

Our analysis will rely on the \emph{dynamic regret} that is related to tracking regret, and has been introduced in~\citep{besbes2015,jadbabaie2015} to quantify the performance of online algorithms. The dynamic regret is defined as (cf. \eqref{eq.sta-reg}) 
\begin{align}\label{eq.dyn-reg}
    {\rm Reg}^{\rm d}_{\cal A}(T):=\sum_{t=1}^T {\cal L}_t(\hat{f}_t(\bbx_t))-\sum_{t=1}^T{\cal L}_t(f_t^*(\bbx_t))
\end{align}
where the benchmark is the aggregate loss incurred by a sequence of the best dynamic
functions$\{f_t^*\}$ from ${\cal F}$ formed by the union of function spaces ${\cal H}_p$ induced by $\{\kappa_p\}$, given by 
\begin{equation}\label{eq.realtime-prob}
 f_t^*(\cdot)   \in\arg\min_{\{f_{p,t}^*,\,p\in{\cal P}\}}\, {\cal L}_t(f_p^*(\bbx_t))~~~{\rm with}~~~f_{p,t}^*(\cdot)\in\arg\min_{f\in{\cal H}_p} \,{\cal L}_t(f(\bbx_t))
\end{equation}
Comparing \eqref{eq.slot-opt} with \eqref{eq.realtime-prob} we deduce that the dynamic regret is always larger than the static regret in \eqref{eq.sta-reg}. Thus, a sub-linear dynamic regret implies a sub-linear static regret, but not vice versa. 
Given $\{{\cal L}_t\}$, AdaRaker generates functions $\{\hat{f}_t\}$ to minimize the dynamic regret. 

To assess the AdaRaker performance, we will start with an \emph{intermediate} result on the static regret associated with any sub-interval $I\subseteq{\cal T}$.

\begin{lemma}\label{lemma7}
Under (as1)-(as3), the static regret on any interval $I\subseteq {\cal T}$ is given by 
	\begin{align}\label{eq.inter-reg}
		{\rm Reg}^{\rm s}_{\cal A}(|I|):=\sum_{t\in I} {\cal L}_t(\hat{f}_t(\bbx_t))-\sum_{t\in I}{\cal L}_t(f_I^*(\bbx_t))
	\end{align}
where $|I|$ denotes the length of interval $I$, and the best time-invariant function approximant is $f_I^*\in\arg\min_{f\in\bigcup_{p\in{\cal P}}{\cal H}_p} \sum_{t\in I} {\cal L}_t(f(\bbx_t))$, with ${\cal H}_p$ denoting the RKHS induced by $\kappa_p$.
Then for any interval $I\subseteq {\cal T}$ and fixed positive constants $C_0$, $C_1$, the following bound holds
	\begin{align}\label{eq.lemma7}
		{\rm Reg}^{\rm s}_{\rm AdaRaker}(|I|)\leq C_0\sqrt{|I|}+C_1\ln T\sqrt{|I|},~{\rm w.h.p.}
	\end{align}
\end{lemma}

\begin{proof}
	 See Appendix \ref{app.pf.lemm7}.
\end{proof}

Lemma \ref{lemma7} establishes that by combining Raker learners with different learning rates, AdaRaker can achieve sub-linear static regret over \emph{any} interval $I$ with arbitrary interval length. This also holds for intervals overlapping with multiple intervals; see e.g., the red interval in Fig. \ref{fig:meta}. 
Clearly, the best fixed solution in \eqref{eq.inter-reg} is interval specific, which can vary over different intervals. This is qualitatively why the function approximants generated by AdaRaker can cope with a \emph{time-varying} benchmark. Such an intuition will in fact become quantitative in the next theorem, which establishes the dynamic regret for AdaRaker.

\begin{customthm}{2}\label{thm.dyn-reg}
		Suppose (as1)-(as3) are satisfied, and define the accumulated variation of online loss functions as
\begin{align}\label{eq.var-loss}
\!{\cal V}(\{{\cal L}_t\}_{t=1}^T):=\sum_{t=1}^T\max_{f\in {\cal F}}\,\big|{\cal L}_{t+1}(f(\bbx_{t+1}))\!-\!{\cal L}_t(f(\bbx_t))\big|
\end{align}
where ${\cal F}:=\bigcup_{p\in{\cal P}}{\cal H}_p$. 
		Then AdaRaker can afford a dynamic regret in \eqref{eq.dyn-reg} bounded by
	\begin{align}\label{eq.dyn-regbound}
		{\rm Reg}^{\rm d}_{\rm AdaRaker}(T)\leq  &(2+C_0+C_1\ln T)T^{\frac{2}{3}}{\cal V}^{\frac{1}{3}}(\{{\cal L}_t\}_{t=1}^T)\nonumber\\
		=&\tilde{{\cal O}}\left(T^{\frac{2}{3}}{\cal V}^{\frac{1}{3}}(\{{\cal L}_t\}_{t=1}^T)\right),~{\rm w.h.p.}
	\end{align}
	where $\tilde{\cal O}$ neglects the lower-order terms with a polynomial $\log T$ rate.
\end{customthm}
\begin{proof}
	See Appendix \ref{app.pf.thm.dyn-reg}.
\end{proof}

Theorem \ref{thm.dyn-reg} asserts that AdaRaker's dynamic regret depends on the variation of loss functions in \eqref{eq.var-loss}, and on the horizon $T$. 
Interesting enough, whenever the loss functions \emph{do not vary on average}, meaning ${\cal V}(\{{\cal L}_t\}_{t=1}^T)=\mathbf{o}(T)$, AdaRaker achieves sub-linear dynamic regret. To this end, it is useful to present an example where this argument holds.

\paragraph{Intermittent switches:}
With ${\cal L}_t\neq{\cal L}_{t+1}$ defining a switch, consider that the number of switches is sub-linear over $T$; that is, $\sum_{t=1}^T \mathds{1}({\cal L}_t\neq{\cal L}_{t+1})=T^{\gamma}$, $\forall \gamma\in[0,1)$.
Then it follows that ${\cal V}(\{{\cal L}_t\}_{t=1}^T)={\cal O}(T^{\gamma})$, since the one-slot variation of the loss functions is bounded.

Other setups with sub-linear accumulated variation emerge, e.g., when the per-slot variation decreases as ${\cal V}({\cal L}_t)={\cal O}(t^{\gamma-1})$, $\forall \gamma\in[0,1)$. Besides dynamic losses, sub-linear dynamic regrets can be also effected by confining the variability of optimal function estimators. 

\begin{customthm}{3}
\label{thm.reg-sa-d}
Suppose the conditions of Theorem \ref{thm.dyn-reg} hold, and define the regret relative to an $m$-switching dynamic benchmark as ${\rm Reg}^{m}_{\cal A}(T)\!:=\!\sum_{t=1}^T\! {\cal L}_t(\hat{f}_t(\bbx_t))-\sum_{t=1}^T{\cal L}_t(\check{f}_t^*(\bbx_t))$, 
	where $\{\check{f}_t^*\}$ is any trajectory from 
	\begin{equation}\label{eq.m-switch}
	\Big\{\textstyle\big\{\check{f}_t^*\big\}_{t=1}^T\!\in \bigcup_{p\in{\cal P}}{\cal H}_p\Big|\sum_{t=1}^T \mathds{1}(\check{f}_t^*\neq\check{f}_{t-1}^*)\leq m\Big\}.	
	\end{equation}
With $C_0$ and $C_1$ denoting some universal constants, it then holds w.h.p. that
\begin{align}\label{eq.reg.d}
	{\rm Reg}^{m}_{\rm AdaRaker}(T)\!\leq\!(C_0\!+\!C_1\ln T)\sqrt{T m}\!=\!\tilde{{\cal O}}\!\left(\!\sqrt{T m}\right).
	\end{align}
\end{customthm}
\begin{proof}
	See Appendix \ref{app.pf.thm.reg-sa-d}.
\end{proof}

Theorem \ref{thm.reg-sa-d} asserts that without prior knowledge of the environment dynamics, the dynamic regret of AdaRaker is sub-linearly growing with time, provided that the number of changes of the optimal function estimators is sub-linear in $T$; that is, ${\rm Reg}^{m}_{\rm AdaRaker}(T)=\mathbf{o}(T)$ given $m=\mathbf{o}(T)$.  
Therefore, our AdaRaker can track the optimal dynamic functions, if the optimal function \emph{varies slowly} over time; e.g., it does not change in the long-term average sense. 
While the conditions to guarantee optimality in dynamic settings may appear restrictive, they are practically relevant, since abrupt changes or adversarial samples will likely not happen at each and every slot in practice. 
\section{Numerical Tests}
\label{sec:test}

This section evaluates the performance of our novel algorithms in online regression tasks using both synthetic and real-world datasets. 

In the subsequent tests, we use the following benchmarks. 

\vspace{0.1cm}
\noindent\textbf{RBF}: the online \emph{single} kernel learning method using Gaussian kernels, a.k.a. radial basis functions (RBFs), with bandwidth $\sigma^2=\{0.1, 1, 10\}$ (cf. RBF01, RBF1, RBF10);

\noindent\textbf{POLY}: the online \emph{single} kernel method using polynomial kernels, with degree $d=\{2, 3\}$ (cf. POLY2, POLY3); 

\noindent\textbf{LINEAR}: the online \emph{single} kernel learning method using a linear kernel; 

\noindent\textbf{AvgMKL}: the online \emph{single} kernel learning method using the average of candidate kernels without updating the weights;

\noindent\textbf{OMKL}: the popular online (O)MKL algorithm without a budget \citep{sahoo2014online};

\noindent\textbf{OMKL-B}: the OMKL algorithm on a budget for regression modified from its single kernel version \citep{kivinen2004}, with the kernel combination weights updated as \eqref{eq.mkl-weight}; 

\noindent\textbf{M-Forgetron}: the online multi-kernel based Forgetron modified from its single kernel version \citep{dekel2008}, with the kernel combination weights updated as in \eqref{eq.mkl-weight};

\noindent\textbf{AdaMKL}: the adaptive version of OMKL that operates in a similar fashion as Algorithm \ref{algo:sadapt}, but instead of using our Raker as an ensemble, it adopts OMKL as an instance $\mathcal{A}_I$. 

\vspace{0.1cm}
{Note that AdaMKL, OMKL-B, and M-Forgetron have not been formally proposed in existing works, but we introduced them here only for comparison purposes. 
All the considered MKL approaches use a dictionary of Gaussian kernels with $\sigma^2=\{0.1, 1, 10\}$, and {AvgMKL}, OMKL, AdaMKL, OMKL-B, and M-Forgetron also include a linear, and a polynomial kernel with order of $2$ into their kernel dictionary. For all MKL approaches, the stepsize for updating kernel combination weights  in \eqref{eq.mkl-weight} is chosen as $0.5$ uniformly, while the stepsize for updating per-kernel function estimators will be specified later in each test.
The regularization parameter is set equal to $\lambda=0.01$ for all approaches.} 
Entries of $\{\bbx_t\}$ and $\{y_t\}$ are normalized to lie in $[0,1]$. 
Regarding AdaMKL and AdaRaker, multiple instances are initialized on intervals with length $|I|:=2^0, 2^1, 2^2, \ldots$, along with the corresponding learning rate on the interval $I$ as $\eta^{(I)}:=\min\{1/2,10/\sqrt{|I|}\}$; see the example in Figure \ref{fig:meta}. 
 {All the results in the tables were reported using the performance at the last time index.}

\begin{table}
\small
{\color{black} 
\begin{center}
\vspace{-0.1cm}
 \begin{tabular}[h]{ c | c | c |c|c|c }
\hline 
    Time index & $[1, 200]$ &$[201, 1000]$  &  $[1001, 2000]$& $[2001, 2300]$ & $[2301, 3000]$\\ \hline
    $\sigma^2$& $0.01$  & $1$ & $10$ &$0.01$  & $1$   \\ \hline\hline
   Time index &$[3001, 3500]$& $[3501, 4300]$ & $[4301, 5100]$ &$[5101, 5900]$& $5901,6500$\\ \hline
    $\sigma^2$& $10$  & $0.01$ & $1$ &$0.01$  & $0.1$  \\ \hline
    \end{tabular}
\end{center}
\vspace{-0.3cm}
\caption{Intervals and $\{\sigma^2\}$ for synthetic dataset.}
\label{tab:syn:sig}}
\vspace{-0.2cm}
\end{table}

\subsection{ Synthetic data tests for regression} This subsection presents the synthetic data tests for regression.

\noindent{\bf Data generation.} In this test, two synthetic datasets were generated as follows.\\
 For \emph{Dataset 1}, the feature vectors $\{\bbx_t\in\mathbb{R}^{10}\}_{t=1}^{14,000}$ are generated from the standardized Gaussian distribution, while $y_t$ is generated as $y_t=\sum_{\tau=1}^t\alpha_{\tau}\kappa_{\tau}(\bbx_t, \bbx_{\tau})$, where {$\{\alpha_t\}$ is generated as $\alpha_t=1+e_t$ with $e_t\sim \mathcal{N}(0,\sigma_{\alpha}^2)$ and $\sigma_{\alpha}=0.01$, while} $\{\kappa_t\}$ are kernel functions that change overtime: for $t\in[1, 8000]\bigcup[18001,26000]$,  $\kappa_t$ is a Gaussian kernel with $\sigma^2=1$, while for $t\in[8001, 18000]\bigcup[26001, 36000]$ the Gaussian kernel has $\sigma^2=10$. 
Therefore, the underlying nonlinear relationship between $\bbx_t$ and $y_t$ undergoes intermittent changes, which come from corresponding changes in the optimal kernel combinations.\\
\emph{Dataset 2} is generated with more variance and switching points.  Specifically, the feature vectors are generated from the standardized Gaussian distribution, while $y_t$ is generated as $y_t=\sum_{\tau=1}^t\alpha_{\tau}\kappa_{\tau}(\bbx_t, \bbx_{\tau})$, where $\{\kappa_t\}$ change over $10$ intervals with different $\sigma^2$; see Table \ref{tab:syn:sig}.

 \begin{figure}[t]
\begin{tabular}{cc}
\hspace*{-6ex}
\includegraphics[width=8.5cm]{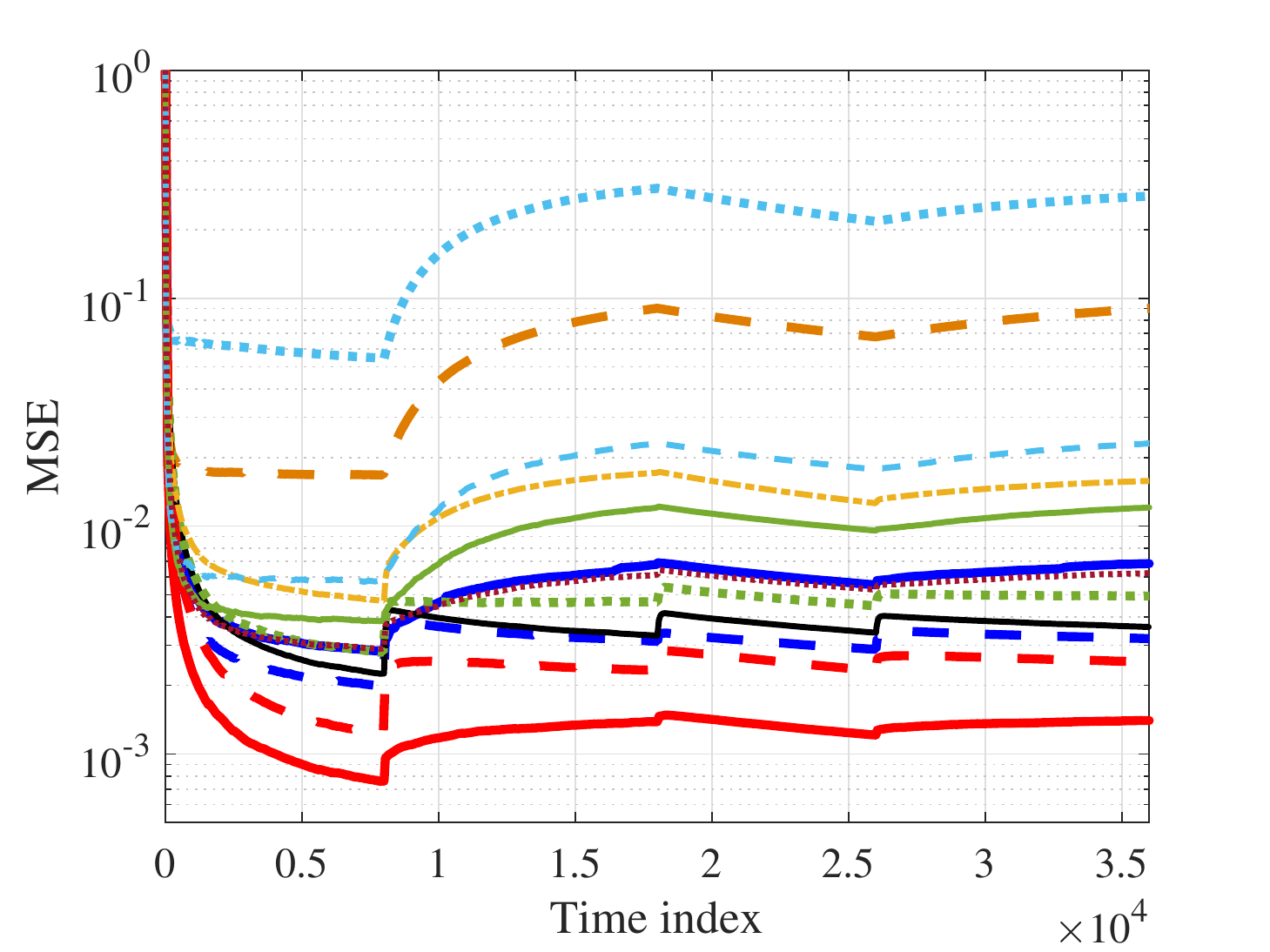}&
\hspace*{-6ex}
\includegraphics[width=8.5cm]{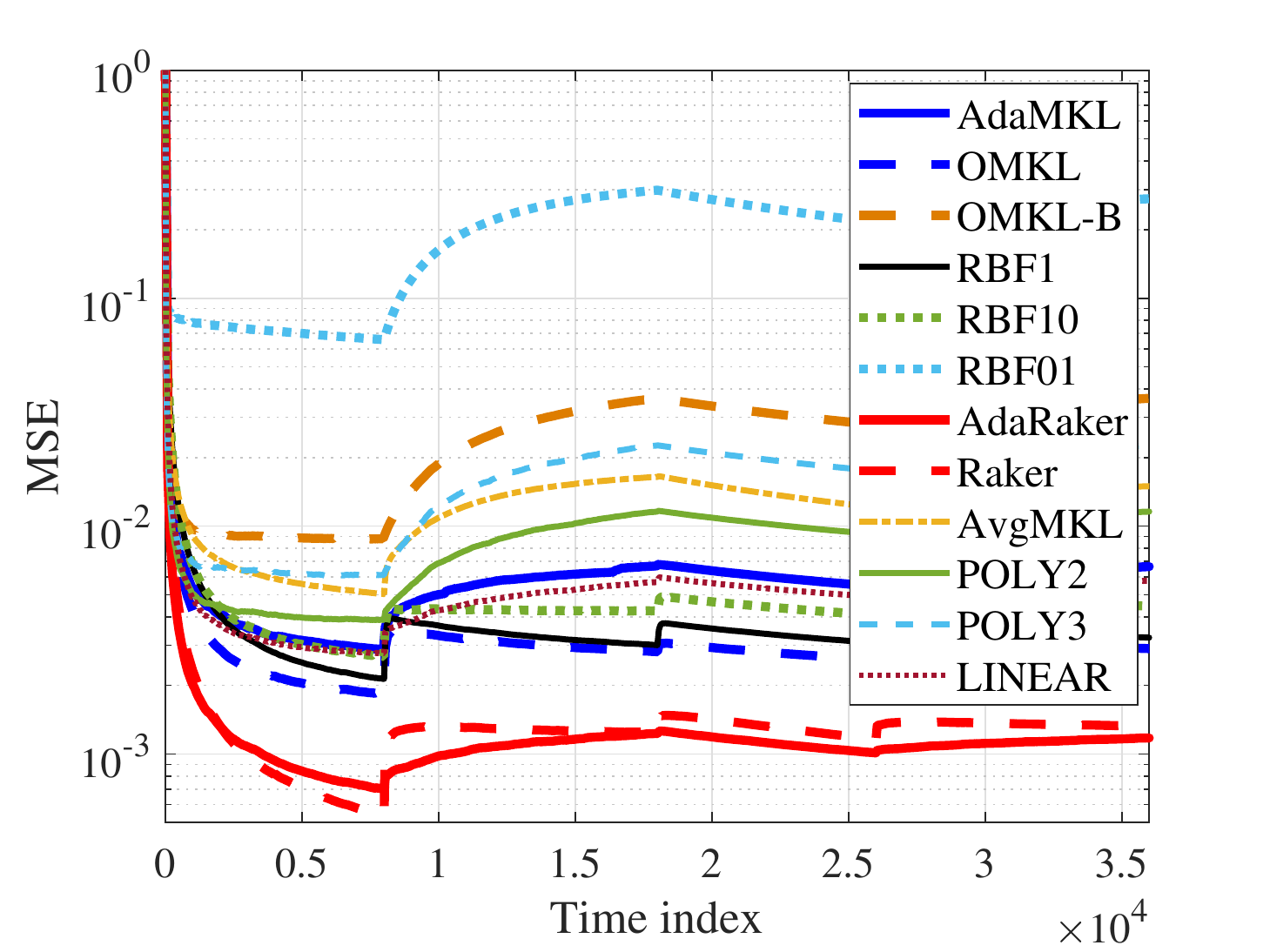}
\\
(a)& (b)
\end{tabular}
\vspace*{-0.2cm}
  \caption{{MSE performance on synthetic Dataset 1: a) $D=B=20$; b) $D=B = 50$.}} 
\label{fig:syn}
\vspace{-0.2cm}
\end{figure}

\begin{figure}[t]
\begin{tabular}{cc}
\hspace*{-6ex}
\includegraphics[width=8.5cm]{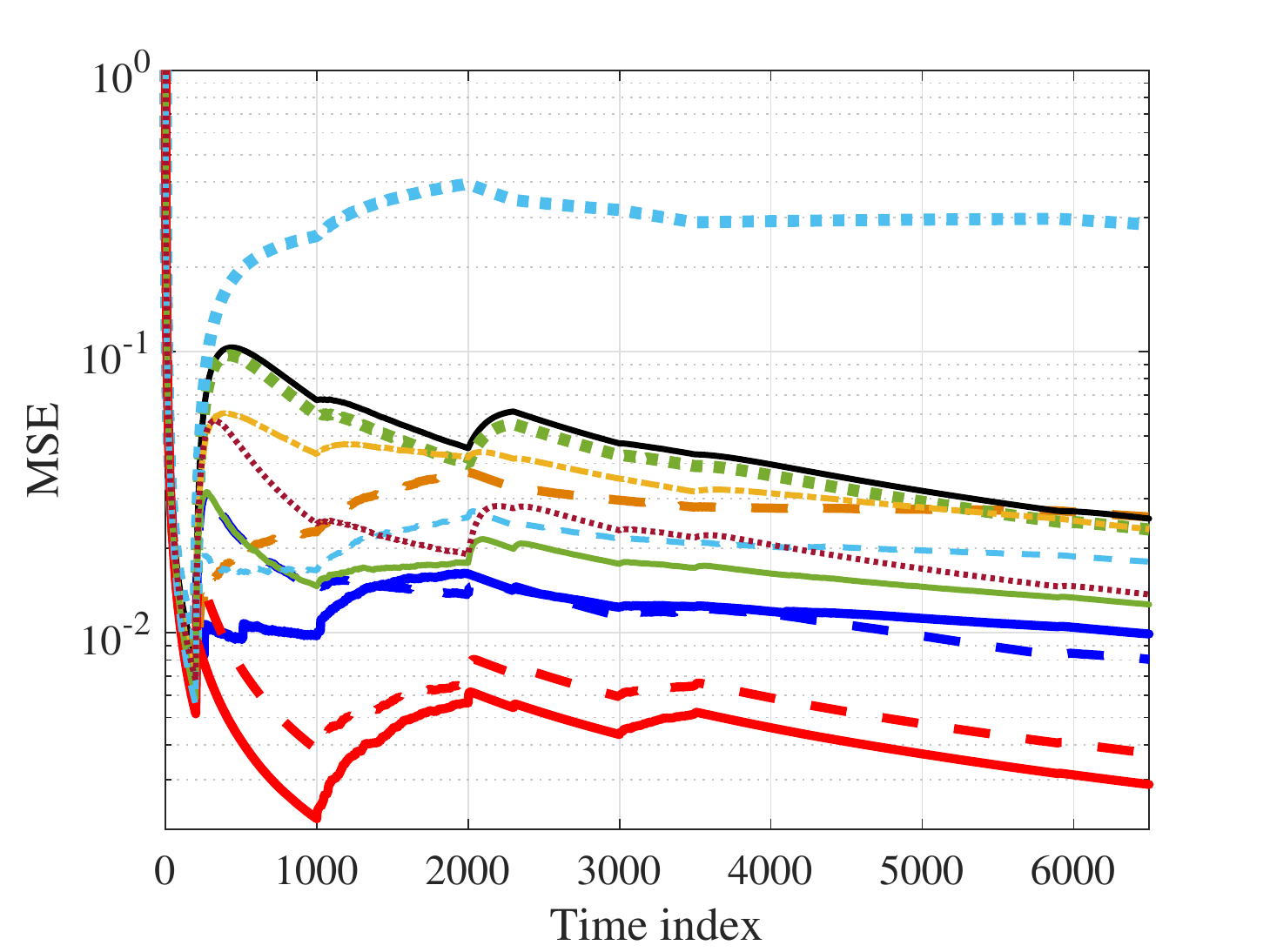}&
\hspace*{-6ex}
\includegraphics[width=8.5cm]{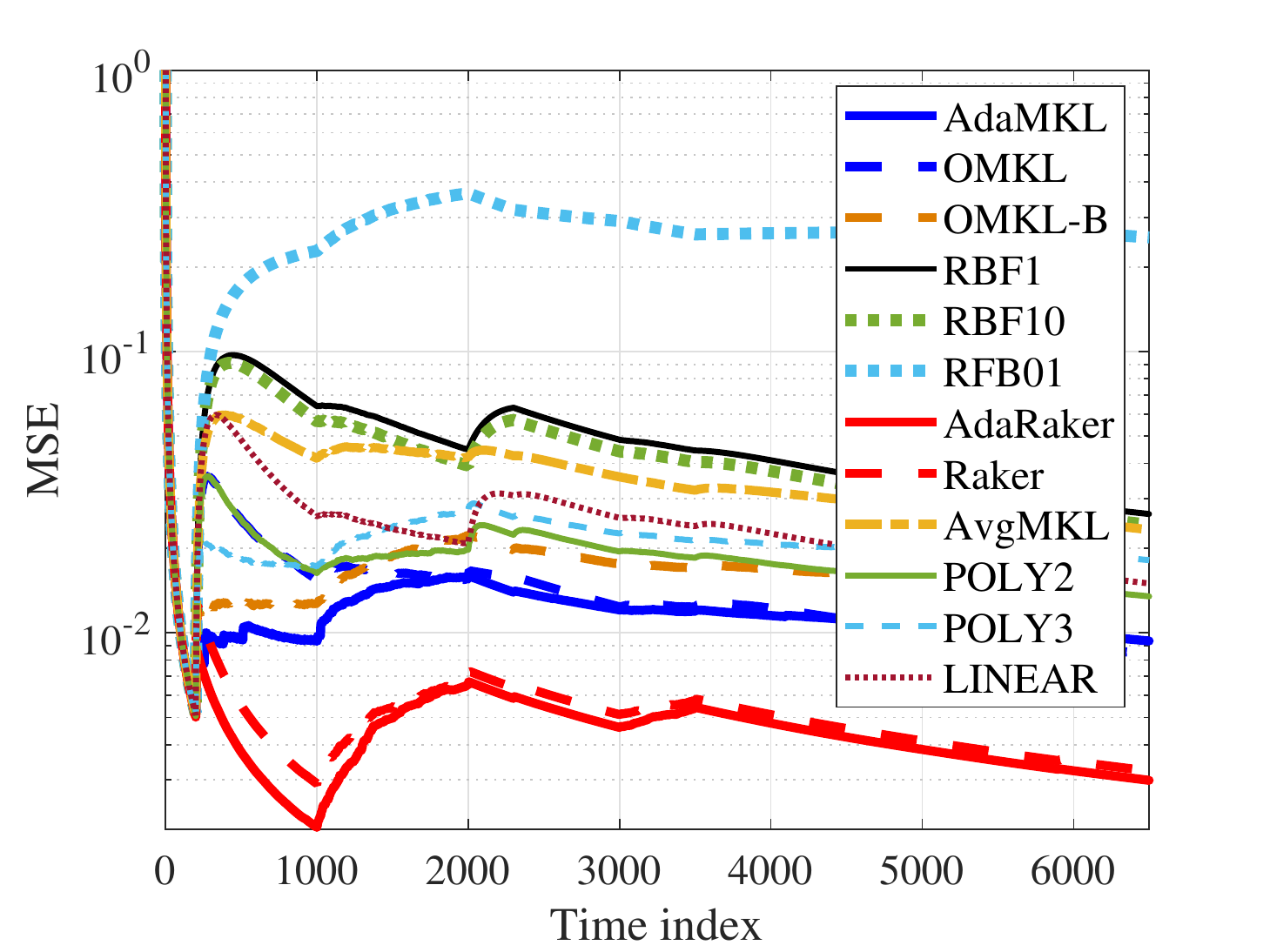}
\\
(a)& (b)
\end{tabular}
\vspace*{-0.2cm}
  \caption{{MSE performance on synthetic Dataset 2: a) $D=B =20$; b) $D=B =50$.}}
\label{fig:syn2}
\vspace{-0.2cm}
\end{figure}

\noindent{\bf Testing performance.} 
The performance of all schemes is tested in terms of the mean-square (prediction) error ${\rm MSE}(t):=(1/t)\sum_{\tau=1}^t{(y_{\tau}-\hat{y}_{\tau})^2}$ in Figure \ref{fig:syn} and Figure \ref{fig:syn2}, and their CPU time is listed in Table \ref{tab:cpu1}. For OMKL-B, $B=20$ and $50$ most recent data samples were kept in the budget; and for RF-based Raker and AdaRaker approaches, $D=20$ and $50$ orthogonal random features were used by default. 
{The default stepsize is chosen as $1/\sqrt{T}$ for RBF, POLY, LINEAR, AvgMKL, OMKL, OMKL-B and Raker.}
In both tests, AdaRaker outperforms the alternatives in terms of MSE, especially when the true nonlinear relationship between $\bbx_t$ and $y_t$ changes; e.g., compare the MSE of KL-RBF and Raker with that of AdaRaker at $t=8000, 18000, 26000$ in Figure \ref{fig:syn}, and $t=200, 2000, 3000, 3500$ in Figure \ref{fig:syn2}. This corroborates the effectiveness of the novel AdaRacker method that can flexibly select learning rates according to the variability of the environments, and adaptively combine multiple kernels when the optimal underlying nonlinear relationship is varying over time. {In addition, MKL approaches including our Raker approach enjoy lower MSE than that of the single-kernel approaches as well as the simple AvgMKL approach,} which is also aligned with our design principle of developing MKL schemes that broaden generalizability of a kernel-based learner over a larger function space.

\begin{table}[t]
{\color{black}
\begin{center}
\small
\vspace{0.2cm}
 \begin{tabu}[h]{ c || c |c| c|c}
\hline 
    &\multicolumn{2}{c|} {\textbf{Dataset 1}}& \multicolumn{2}{|c} {\textbf{Dataset 2}} \\ \hline 
          \textbf{Setting} & $D=B=20$ & $D=B=50$ & $D=B=20$ & $D=B=50$\\ \hline \hline 
    AdaMKL & \multicolumn{2}{c|} {$318.52$} & \multicolumn{2}{|c} {$27.29$} \\ \hline
   OMKL & \multicolumn{2}{c|} {$157.10$} & \multicolumn{2}{|c} {$5.47$} \\ \hline
    RBF & \multicolumn{2}{c|} {$47.83$} & \multicolumn{2}{|c} {$1.06$}\\ \hline
\rowfont{\color{black}} POLY2 & \multicolumn{2}{c|} {$6.01$} & \multicolumn{2}{|c} {$0.47$}\\ \hline
\rowfont{\color{black}} POLY3 & \multicolumn{2}{c|} {$28.27$} & \multicolumn{2}{|c} {$1.24$}\\ \hline
\rowfont{\color{black}} LINEAR & \multicolumn{2}{c|} {$4.80$} & \multicolumn{2}{|c} {$0.35$}\\ \hline
 \rowfont{\color{black}} AvgMKL & \multicolumn{2}{c|} {$144.85$} & \multicolumn{2}{|c} {$5.02$}\\ \hline
    OMKL-B&$3.75$  & $4.05$ &$0.72$  & $0.77$\\ \hline
    Raker&$1.39$  & $1.53$ &$0.18$  & $0.20$\\ \hline
    AdaRaker&$21.94$  & $24.24$ &$3.32$  & $3.54$  \\ \hline
      \end{tabu}
\end{center}}
\vspace{-0.2cm}
\caption{CPU time (in seconds) on synthetic datasets. RBF, POLY represents all single-kernel methods using RBF and polynomial kernels, since they have the same CPU time.}
\label{tab:cpu1}
\end{table}

{Table \ref{tab:cpu1} records the CPU time of all benchmark algorithms running tests on two different datasets.} It can be observed that leveraging the RF-based approximation, the proposed AdaRaker and Raker algorithms are much faster than AdaMKL and OMKL; hence, they are preferable especially for large-scale datasets. Although the CPU time of OMKL-B with a budget size $B=20$ or $B=50$ is relatively low, OMKL-B does not perform as well as AdaRaker and Raker algorithms. Therefore, the AdaRaker and Raker approaches attain a sweet-spot in the performance-complexity tradeoff.

\begin{table}
\small
{\color{black} 
\begin{center}
\vspace{0.2cm}
 \begin{tabular}[h]{ c || c | c |c }
\hline 
     \textbf{Dataset} & \textbf{\# features} ($d$) & \textbf{\# samples} ($T$)& \textbf{feature type} \\ \hline \hline
    Twitter & $77$ &$14,000$  & real \& integer \\ \hline
    Twitter (Large)& $77$ & $100,000$ & real \& integer\\ \hline
    Tom's hardware& $96$  & $10,000$ &  real \& integer \\ \hline
    Energy& $27$ & $18,600$& real \\ \hline
    Air quality&$13$  & $9,358$ & real \\ \hline
      \end{tabular}
\end{center}
\vspace{-0.3cm}
\caption{A summary of real datasets used in the tests.}
\label{tab:1}
\vspace{-0.2cm}
}
\end{table}

{\subsection{Real data tests for online regression}}
 
To further evaluate our algorithms in real-world scenarios, the present subsection is devoted to testing and comparing on several popular real datasets. 

\noindent\textbf{Datasets description.} Performance is tested on benchmark datasets from UCI machine learning repository \citep{Lichman:2013}. 

\begin{itemize}\itemsep0.08em 
	\item \textbf{Twitter} dataset consists of $T=14, 000$ samples from a popular micro-blogging platform Twitter, where $\bbx_t\in\mathbb{R}^{77}$ include features such as the number of new interactive authors, and the length of discussion on a given topic, while $y_t$ represents the average number of active discussion (popularity) on a certain topic \citep{kawala2013predictions}. {A larger  dataset with $T= 100,000$  is also included for testing only (Ada)Raker and OMKL-B, since other methods do not scale to such a large $T$.}
	\item \textbf{Tom's hardware} dataset contains $T=10,000$ samples from a worldwide new technology forum, where a $96$-dimensional feature vector includes the number of discussions involving a certain topic, while $y_t$ represents the average number of display about a certain topic on  Tom's hardware \citep{kawala2013predictions}. 
	\item \textbf{energy} dataset consists of $T=18,600$ samples, with each $\bbx_t\in\mathbb{R}^{27}$ describing the humidity and temperature indoors and outdoors, while $y_t$ denotes the energy use of light fixtures in the house \citep{candanedo2017data}. 
	\item \textbf{air quality} dataset collects $T=9,358$ instances of hourly averaged responses from five chemical sensors located in a polluted area of Italy. 
	The averaged sensor response $\bbx_t\in \mathbb{R}^{13}$ contains the hourly concentrations of  e.g., CO, Non Metanic Hydrocarbons, and Nitrogen Dioxide (NO2), where the goal is to predict the concentration of polluting chemicals $y_t$ in the air \citep{de2008field}. 
\end{itemize}
  
To highlight the effectiveness of our approaches, the datasets mainly include time series data, where non-stationarity is more likely to happen; see Table \ref{tab:1} for a summary. 

\begin{table}
\small
\begin{center}
\vspace{0.1cm}
 \begin{tabu}[h]{ c || c|c | c |c }
\hline 
\textbf{Algorithms/ Datasets}&~Twitter~ &~ Tom's ~& Energy&~ Air ~ \\ \hline \hline
     RBF ($\sigma^2=0.1$) & $27.0$ & $14.4$ &$28.9$  & $26.3$   \\ \hline
    RBF ($\sigma^2=1$) &$13.5$  &  $17.0$& $28.8$ & $12.7$ \\ \hline
     RBF ($\sigma^2=10$) & $23.3$&  $18.8$ &$28.8$ & $15.5$  \\ \hline
    \rowfont{\color{black}} POLY2&$12.7$&$22.3$&$28.8$&$7.34$ \\ \hline
    \rowfont{\color{black}} POLY3&$20.4$&$22.7$&$28.9$& $5.91$ \\ \hline
    \rowfont{\color{black}} LINEAR &$8.57$&$19.5$&$28.8$&$10.7$ \\ \hline
    \rowfont{\color{black}} AvgMKL &$14.4$&$17.5$&$28.7$& $11.9$\\ \hline
    OMKL&$8.55$   &  $14.3$ &$28.1 $  & $6.4$ \\ \hline
     AdaMKL& $16.1$& $18.4$& $30.4$& $10.1$ \\ \hline
     OMKL-B ($B=50$) & $27.0$ &$22.1$&$73.3$ &$35.9$   \\ \hline
    Raker ($D=50$) & $3.0$ & $3.4$& $19.3$& $2.0$ \\ \hline
    AdaRaker ($D=50$) & $\mathbf{2.6}$ &$\mathbf{1.9}$ & $\mathbf{13.8}$ &$\mathbf{1.3}$   \\ \hline
    \end{tabu}
\end{center}
\vspace{-0.3cm}
\caption{MSE ($10^{-3}$) performance of different algorithms with stepsize $1/\sqrt{T}$.}
\label{tab:mse1}
\end{table}

\begin{table}
\small
\begin{center}
\vspace{0.1cm}
 \begin{tabu}[h]{ c || c|c | c |c }
\hline 
\textbf{Algorithms/ Datasets}&~Twitter~ &~ Tom's ~& Energy&~ Air ~ \\ \hline \hline
     RBF ($\sigma^2=0.1$) & $17.2$ & $3.3$ &$16.6$  & $8.1$    \\ \hline
         RBF ($\sigma^2=1$) &$3.3$  &  $5.1$& $16.4$ & $2.8$ \\ \hline
     RBF ($\sigma^2=10$) & $5.6$&  $13.6$ &$16.4$ & $18.9$  \\ \hline
      \rowfont{\color{black}} POLY2&$8.1$&$15.9$&$16.2$&$3.3$ \\ \hline
    \rowfont{\color{black}} POLY3&$20.4$&$20.7$&$16.2$&$4.6$ \\ \hline
    \rowfont{\color{black}} LINEAR&$2.7$&$4.8$&$16.3$& $2.9$\\ \hline
    \rowfont{\color{black}} AvgMKL &$7.1$&$6.2$&$16.3$&$2.8$ \\ \hline
    OMKL & $4.2$   & $3.3$ &$16.2 $  & $2.4$ \\ \hline
     AdaMKL& $16.1$& $18.4$& $30.4$& $10.1$ \\ \hline
     OMKL-B ($B=50$) & $9.9$ &$11.8$&$19$ &$7.1$    \\ \hline
    Raker ($D=50$) & $2.9$ & $2.6$& $\mathbf{13.8}$& $\mathbf{1.3}$ \\ \hline
        AdaRaker ($D=50$) & $\mathbf{2.6}$ &$\mathbf{1.9}$ & $\mathbf{13.8}$ &$\mathbf{1.3}$  \\ \hline
    \end{tabu}
\end{center}
\vspace{-0.3cm}
\caption{MSE ($10^{-3}$) performance of different algorithms with optimally chosen stepsizes.}
\label{tab:mse:opt}
\vspace{-0.1cm}
\end{table}

\noindent\textbf{MSE performance.} 
 The MSE performance of each algorithm on the aforementioned datasets is presented in Table \ref{tab:mse1}. 
 By default, we use the complexity $B=D=50$ for OMKL-B and (Ada)Raker, and the stepsize {\small$1/\sqrt{T}$} for RBF, {POLY, LINEAR, AvgMKL}, OMKL, OMKL-B and Raker. To boost the performance of each algorithm, their MSE when using manually tuned stepsizes is also reported in Table \ref{tab:mse:opt}, which selects the best stepsize on each dataset among $\{10^{-3},10^{-2},\cdots,10^3\}/\sqrt{T}$. 
A common observation is that leveraging the flexibility of multiple kernels, MKL methods in most cases outperform the algorithms using only a single kernel. 
{By simply averaging over all the kernels, AvgMKL outperforms most of single kernel methods, but performs worse than the adaptive kernel combination methods.} 
This confirms that relying on a pre-selected kernel function is not sufficient to guarantee low fitting loss, while allowing the MKL approaches to select the best kernel combinations in a data-driven fashion holds the key for improved performance.

In most tested datasets,
Raker obtains function approximants with lower MSE relative to MKL alternatives without RF approximation. 
Furthermore, incorporating multiple Raker instances with variable learning rates, AdaRaker consistently yields the lowest MSE in all the tests.
As it has been shown in the synthetic data test, the sizable performance gain of AdaRaker appears when the underlying nonlinear models change in the tested time-series datasets. 
This observation is aligned with our design principle of AdaRaker; that is, when the optimal function predictor varies slowly (fast), AdaRaker tends to select a Raker instance with small (large) learning rate. 
Interesting enough, even with adaptive learning rate, AdaMKL does not perform as well as OMKL in some tests. This is partially because unlike AdaRaker with fixed number of RFs, each instance in AdaMKL involves a different \emph{number of support vectors} (samples). The instance operating on the longest interval contains at most $T/2$ support vectors, which may deteriorate performance relative to OMKL with $T$ support vectors.

\begin{table}
\small
{\begin{center}
 \begin{tabular}[h]{ c || c | c |c|c||c|c|c|c||c }
\hline 
\textbf{MSE} &\multicolumn{4}{|c||}{\textbf{OMKL-B}}& \multicolumn{4}{|c||}{\textbf{Raker}}& \!\!\textbf{AdaRaker}\!\!\\ \hline 
\!\!\textbf{Stepsize}\!\!&\!\!$1/\sqrt{T}$\!\! &\!\! $0.5/\sqrt{t}$ \!\!& \!\!$0.1/\sqrt{t}$\!\!& \!\!Tuned \!\!& \!\!$1/\sqrt{T}$ \!\! &\!\! $0.5/\sqrt{t}$ \!\!& \!\!$0.1/\sqrt{t}$ \!\!&\!\! Tuned \!\!& \!\! / \!\!\\ \hline \hline
    Twitter & $27.0$ &$27.1$  &  $29.6$& $9.9$ & $3.0$ & $17.9$ &$4.3$& $2.9$ &$\mathbf{2.6}$\\ \hline
    Tom's& $22.1$  & $22.1$ & $22.6$ & $11.8$&$3.4$  & $2.0$ &$7.6$ & $2.6$ &$\mathbf{1.9}$  \\ \hline
    Energy& $73.3$ & $74.1$&  $79.5$ & $19.0$ &$19.3$ & $29.5$ & $25.1$ &$\mathbf{13.8}$&$\mathbf{13.8}$ \\ \hline
    Air &$35.9$  & $35.9$ &  $40.1$ &$7.1$ &$2.0$  & $29.1$&$4.0$ &$\mathbf{1.3}$ &$\mathbf{1.3}$\\ \hline
 \!\!\!\!  Twitter (Large) \!\!\!\!& $20.7$ & $27.2$ & $28.0$ & $11.3$&$3.2$ & $3.1$ & $3.3$& $3.0$ & $\mathbf{2.7}$\\ \hline
    \end{tabular}
\end{center}
\vspace{-0.3cm}
\caption{MSE ($10^{-3}$) versus the choice of stepsizes with complexity $B=D=50$.}
\label{tab:2}}
\end{table}

{Table \ref{tab:2} further compares the MSE performance of AdaRaker with OMKL-B and Raker using different  stepsizes. 
Clearly, the performance of OMKL-B and Raker is sensitive to the choice of stepsizes. While the optimal stepsize varies from dataset to dataset, selecting a constant stepsize {\small$1/\sqrt{T}$} generally leads to better performance than a diminishing one of ${\cal O}(1/\sqrt{t})$. 
In the online scenarios however, the choice  {\small$1/\sqrt{T}$} may not be feasible if $T$ is unknown ahead of time. 
In contrast, AdaRaker obtained the best MSE performance without knowing $T$, and without the need of stepsize selection, which confirms that AdaRaker is capable of adapting its stepsize to variable environments with unknown dynamics.
}

\noindent\textbf{Computational complexity.} The CPU time of all the considered schemes is recorded under all the tests; see Table \ref{tab:3}. 
It is evident that in all tests, our RF-based MKL methods including Raker and AdaRaker are computationally more efficient than other MKL methods except that OMKL-B is faster than AdaRaker. 
Intuitively speaking, the per-slot complexity of Raker does not grow with time, since it  requires computing only one inner product of two $2D$-dimensional vectors per kernel learner, while the computational complexity of AdaMKL, OMKL, {POLY, LINEAR, AvgMKL}, and RBF increases with time at least linearly.
With a fixed budget size, OMKL-B enjoys light-weight updates that leads to a lower CPU time than alternatives, but higher than Raker. 
However, given such a limited budget of data, OMKL-B exhibits higher MSE than AdaRaker and Raker; see MSE in Tables \ref{tab:mse1} and \ref{tab:2}.

 \begin{table}[t]
 \small
 \begin{center}
    \begin{tabu}[h]{ c || c|c | c |c }
\hline 
\textbf{Algorithms/ Datasets}&~Twitter~ &~ Tom's ~& Energy&~ Air ~ \\ \hline \hline
     RBF & $54.7$&  $32.5$ &$26.6$ & $1.71$  \\ \hline
     \rowfont{\color{black}} POLY2&$2.25$&$1.16$&$2.42$&$0.58$ \\ \hline
    \rowfont{\color{black}} POLY3&$5.62$&$2.97$&$7.90$&$1.51$ \\ \hline
    \rowfont{\color{black}} LINEAR&$1.83$&$0.98$&$196$& $0.39$\\ \hline
    \rowfont{\color{black}} AvgMKL &$148.4$&$81.6$&$82.4$&$5.29$ \\ \hline    OMKL&$153.5$   &  $81.9$ &$83.3 $  & $5.90$  \\ \hline
         AdaMKL& $164.1$& $102.7$& $117.9$& $35.3$ \\ \hline
    OMKL-B ($B=50$) & $1.89$ &$1.42$&$2.02$ &$0.89$  \\ \hline
    Raker ($D=50$) & $\mathbf{0.51}$ & $\mathbf{0.38}$& $\mathbf{0.65}$& $\mathbf{0.28}$
    \\ \hline
    AdaRaker ($D=50$) & $8.64$ &$6.03$ & $10.94$ &$5.28$  \\ \hline
    \end{tabu}
\end{center}
\vspace{-0.3cm}
\caption{A summary of CPU time (second) on real datasets.}
\label{tab:3}
\end{table}

Running multiple instances of Raker in parallel, the complexity of AdaRaker is reasonably higher than Raker (roughly $\log T$ times higher), but its runtime is still only around $10\%$ of that of AdaMKL, and significantly lower than other single-kernel alternatives especially when the actual feature dimension $d$ is higher than the number of random features $D$.
The computational advantage of our MKL algorithms in this test also corroborates the quantitative analysis at the end of Section \ref{sec.rf-mkl}. 
{Regarding the tradeoff between learning accuracy and complexity, a delicate comparison among OMKL-B, Raker and AdaRaker follows next.}

 \begin{table}[t]
 \small
\begin{center}
\vspace{-0.1cm}
{\color{black}
 \begin{tabular}[h]{ c || c | c |c|c|c|c|c|c|c }
\hline 
\textbf{MSE}&\multicolumn{3}{|c|}{\textbf{OMKL-B}}& \multicolumn{3}{|c}{\textbf{Raker}}& \multicolumn{3}{|c}{\textbf{AdaRaker}}\\ \hline 
\textbf{Complexity}&~ $10$ ~ &~ $50$ ~& $100$&~ $10$ ~&~ $50$ ~& $100$ &~ $10$ ~&~ $50$ ~& $100$ \\ \hline \hline
    Twitter & $28.9$ &$27.0$  &  $26.1$& $5.9$ & $3.0$ &$3.0$& $3.8$ & $\mathbf{2.6}$&$\mathbf{2.6}$\\ \hline
    Tom's& $22.7$  & $22.1$ & $21.7$ &$8.1$  & $3.4$ &$2.3$ & $7.0$  & $1.9$ &$\mathbf{1.8}$  \\ \hline
    Energy& $79.1$ & $73.3$&  $67.9$ &$25.7$ & $19.3$ & $16.4$ & $18.7$ & $13.8$ & $\mathbf{13.3}$\\ \hline
    Air pollution&$36.7$  & $35.9$ &  $35.8$ &$10.1$  & $2.0$&$1.7$ &$4.3$  & $1.3$&$\mathbf{1.2}$\\ \hline
    Twitter (Large) &$25.0$ & $20.7$& $19.0$ &$3.9$ & $3.2$& $3.0$& $3.3$ & $\mathbf{2.7}$& $\mathbf{2.7}$  \\ \hline
    \end{tabular}
    }
\end{center}
\vspace{-0.3cm}
\caption{MSE ($10^{-3}$) versus complexity. For OMKL-B, the complexity measure is the data budget $B$; and for (Ada)Raker, the complexity measure is the number of RFs $D$.}
\label{tab:4}
\end{table}

\begin{table}[t]
\small
\begin{center}
{\color{black}
\begin{tabular}[h]{ c || c | c |c|c |c |c|c|c|c }
\hline 
\textbf{Time} &\multicolumn{3}{|c|}{\textbf{OMKL-B}}& \multicolumn{3}{|c}{\textbf{ Raker}}& \multicolumn{3}{|c}{\textbf{AdaRaker}} \\ \hline \textbf{Complexity}
     &~ $10$ ~&~ $50$ ~& $100$ &~ $10$ ~&~ $50$ ~& $100$&~ $10$ ~&~ $50$ ~& ~$100$~ \\ \hline \hline
    Twitter & $1.42$ &$1.89$  &  $2.84$& $0.42$ & $0.51$ &$0.80$& $7.58$ & $8.64$ &$11.65$\\ \hline
    Tom's& $1.00$  & $1.42$ & $2.81$& $0.41$ &$0.38$  &$0.56$ & $5.09$  & $6.03$ &$8.98$  \\ \hline
    Energy& $1.84$ & $2.02$&  $2.32$& $0.58$ & $0.65$ & $0.76$ & $9.96$ & $10.94$ & $12.47$\\ \hline
    Air pollution&$0.82$  & $0.89$ &  $0.97$&$0.24$  & $0.28$&$0.32$ &$4.09$  & $5.28$&$5.29$\\ \hline
    Twitter (Large) & $12.90$ & $16.34$ &$23.6$ & $6.07$ &$6.63$& $8.42$ & $67.10$ & $78.10$& $109.40$  \\ \hline
    \end{tabular}}
\end{center}
\vspace{-0.3cm}
\caption{CPU time (second) versus complexity of $B$ for OMKL-B, and $D$ for (Ada)Raker.}\label{tab:5}
\end{table}

\noindent\textbf{Accuracy versus complexity.} To further understand the tradeoff between complexity and learning accuracy, the performance of three scalable methods AdaRaker, Raker and OMKL-B is tested under different parameter settings, e.g., $D$, the number of random features, and $B$, the number of budgeted data. 
The MSE performance is reported in Table \ref{tab:4} after one pass of all data in each dataset, while the corresponding CPU time is in Table \ref{tab:5}. 

Not surprising, all three algorithms require longer CPU time as the complexity (in terms of $B$ or $D$) increases. 
For given complexity (same $B$ and $D$), Raker requires the lowest CPU time, and its MSE is also markedly lower than that of OMKL-B in all tests. 
{On the other hand, AdaRaker always attains the lowest MSE, and its performance gain is remarkable especially in the Energy and Air pollution datasets. 
For Twitter (Large) dataset, the performance of AdaRaker does not improve as RFs increase from $D=50$ to $D=100$, which implies that $D=50$ is enough to provide reliable kernel approximation in this dataset.} 
Considering that AdaRaker is embedded with concurrent $\log t$ Raker instances at time $t$, its CPU time is relatively higher. However, one would expect a major reduction in the number of concurrent instances and thus markedly lower CPU time, if a larger basic interval size (instead of base number $2$ in Figure \ref{fig:meta}) is incorporated in AdaRaker real implementation. 

At this point, one may wonder how many RFs are enough for Raker and AdaRaker to guarantee the same online learning accuracy as that of OMKL-B with $B$ samples. While this intriguing question has been recently studied in the batch setting with an answer of $D={\cal O}(\sqrt{B})$ RFs \citep{rudi2017}, its thorough treatment in the online setting constitutes our future research.


  \begin{table}[t]
 \small
\begin{center}
\vspace{-0.1cm}
 \begin{tabu}[h]{ c || c| c | c |c|c|c}
\hline 
&\multicolumn{3}{c|}{\textbf{Classification error}}& \multicolumn{3}{c}{\textbf{CPU time}}\\\hline
\textbf{Algorithms/Datasets}&Movement&Devices& ~Activity~&Movement&Devices&~Activity~\\ \hline \hline
     RBF ($\sigma^2=0.1$)& $43.1$  & $6.67$ & $0.46$ &$2.76$ &$2.13$  & $4.42$   \\ \hline
     RBF ($\sigma^2=1$) & $41.3$ &$28.1$  &$5.21$ & $2.79$& $2.04$ &$4.42$ \\ \hline
     RBF ($\sigma^2=10$)& $40.3$ & $31.8$& $41.1$ & $2.62$ &$2.09$& $4.43$\\ \hline
      \rowfont{\color{black}} POLY2&$43.5$&$14.3$&$3.13$&$1.63$&$0.31$& $0.81$ \\ \hline
    \rowfont{\color{black}} POLY3&$43.6$&$25.2$&$2.39$&$4.75$&$0.57$ & $1.79$\\ \hline
    \rowfont{\color{black}} LINEAR&$43.8$&$47.4$&$4.30$& $1.26$&$0.21$& $0.60$\\ \hline
    \rowfont{\color{black}} AvgMKL  &$41.7$&$23.9$&$3.16$&$10.2$&$6.72$& $14.67$ \\ \hline
    OMKL&$38.2$  & $16.0$ &$0.60$ & $10.27$ &$7.07$ & $14.46$ \\ \hline
         AdaMKL & $3.5$ & $0.86$ & $0.98$ &$33.77$ &$10.51$& $21.46$ \\ \hline
    M-Forgetron ($B=50$)& $1.64$ & $0.53$ &$1.14$&$0.92$&$0.27$ & $0.53$  \\ \hline
    Raker ($D=50$) &$9.74$ & $2.54$&$0.58$ &$\mathbf{0.40}$& $\mathbf{0.12}$ & $\mathbf{0.21}$\\ \hline
    AdaRaker ($D=50$)& $\mathbf{1.10}$ & $\mathbf{0.36}$ & $\mathbf{0.34}$&$6.76$ & $1.73$ & $3.56$\\ \hline
    \end{tabu}
\end{center}
\vspace{-0.2cm}
\caption{Classification error ($\%$) and runtime (second) of different algorithms with the default stepsize $1/\sqrt{T}$ for RBF, OMKL and Raker, and with complexity $B=D=50$.} \label{tab:class1}
\end{table}

\subsection{Real data tests for online classification}
In  this  section, the performance of Raker and AdaRaker is tested on real datasets for the online classification task. We use the logistic loss as the learning objective function with the regularization parameter $\lambda=0.005$ for all considered approaches except for the perceptron-based Forgetron algorithm. Kernels and all other parameters such as the default stepsizes, are chosen as those in the regression task.

  \begin{table}[t]
 \small
\begin{center}
\vspace{-0.1cm}
 \begin{tabu}[h]{ c || c| c | c }
\hline 
&\multicolumn{3}{c}{\textbf{Classification error}}\\\hline
\textbf{Algorithms/Datasets}&~Movement~&~Devices~& ~Activity~\\ \hline \hline
     RBF ($\sigma^2=0.1$)& $28.9$  & $5.16$ & $0.26$    \\ \hline
     RBF ($\sigma^2=1$) & $1.27$ &$0.42$  & $0.53$ \\ \hline
     RBF ($\sigma^2=10$)& $\mathbf{1.10}$ & $0.36$&  $1.14$\\ \hline
       \rowfont{\color{black}} POLY2&$8.19$&$1.7$&$0.56$ \\ \hline
    \rowfont{\color{black}} POLY3&$15.2$&$17.3$&$0.45$ \\ \hline
    \rowfont{\color{black}} LINEAR&$7.46$&$30.7$&$0.60$\\ \hline
    \rowfont{\color{black}} AvgMKL &$1.69$&$2.26$&$0.48$ \\ \hline
    OMKL&$\mathbf{1.10}$  & $0.36$ & $0.29$\\ \hline
     AdaMKL & $3.50$ & $0.86$ & $1.00$\\ \hline
   M-Forgetron ($B=50$)& $1.64$ & $0.53$ & $1.14$ \\ \hline
    Raker ($D=50$) &$\mathbf{1.10}$ & $\mathbf{0.28}$& $\mathbf{0.24}$ \\ \hline
    AdaRaker ($D=50$)& $\mathbf{1.10}$ & $0.36$ &$0.34$\\ \hline
    \end{tabu}
\end{center}
\vspace{-0.2cm}
\caption{Classification error ($\%$) of different algorithms with the dataset-specific optimally chosen stepsizes for RBF, OMKL and Raker, and with complexity $B=D=50$.} \label{tab:class1-opt}
\end{table}

\noindent\textbf{Datasets description.} We test classification performance on the following datasets.
\vspace{-0.2cm}

\begin{itemize}\itemsep0.08em 
 \item\textbf{Movement} dataset consists of $T=13,197$ temporal streams of received signal strength (RSS) measured between the nodes of a wireless sensor network, with each $\bbx_t\in\mathbb{R}^{4}$  comprising $4$ anchor nodes \citep{bacciu2014experimental}. Data has been collected during user movements at the frequency of $8$ Hz ($8$ samples per second). The RSS samples in the dataset have been rescaled to lie in $[-1,1]$. The binary label $y_t$ indicates whether the user's trajectory will lead to a change in the spatial context (here a room change) or not. 
 
  \item\textbf{Electronic Device} dataset consists of $T=3,600$ samples collected as part of a government sponsored study called `Powering the Nation,' where the feature vectors $\bbx_t \in \mathbb{R}^{60}$ represent electricity readings from different households over $15$ mins, sampled within a month \citep{lines2011classification}.  Binary label $y_t$ represents the type of electronic devices used at the certain interval of time time: dishwasher or kettle.
 
\item\textbf{Human Activity} dataset consists of $T=7,352$ samples collected from a group of 30 volunteers wearing a smartphone (Samsung Galaxy S II) on their waist to monitor activities \citep{anguita2013}. Feature vectors $\{\bbx_t \in \mathbb{R}^{30}\}$ here measure e.g., triaxial acceleration and angular velocity, while binary label $y_t$ represents the activity during a certain period: walking or not walking.
 \end{itemize}
 \vspace{-0.1cm}

 \begin{table}[t]
\small
\begin{center}
 \begin{tabular}[h]{ c || c |c | c |c||c|c |c|c||c }
\hline 
&\multicolumn{4}{|c||}{\textbf{OMKL}}& \multicolumn{4}{|c||}{\textbf{Raker}}& \textbf{AdaRaker}\\ \hline 
\textbf{Stepsize}&$1/\sqrt{T}$ & $1/\sqrt{t}$ & $10/\sqrt{t}$& Tuned &$1/\sqrt{T}$  & $1/\sqrt{t}$ & $10/\sqrt{t}$& Tuned & /  \\ \hline \hline
    Movement & $38.2$ &$39.5$  &  $22.3$& $1.10$& $12.1$ & ${8.60}$ &$1.79$& $1.10$& $\mathbf{1.10}$\\ \hline
    Devices& $16.0$  & $13.2$ & $6.06$& $0.36$&$2.54$  & $2.04$ &$0.53$& $\mathbf{0.28}$& $0.36$  \\ \hline
    Activity& $0.60$  & $0.53$ & $0.50$& $0.29$&$0.58$  & $0.52$ &$0.54$& $\mathbf{0.24}$ & $0.34$  \\ \hline
        \end{tabular}
\end{center}
\vspace{-0.3cm}
\caption{Classification error ($\%$)  versus different choices of stepsizes with  $B=D=50$.}
\label{tab:class:st}
\end{table}

 \begin{table}[t]
 \small
\begin{center}
 \begin{tabular}[h]{ l || c | c |c|c|c|c}
\hline 
&\multicolumn{3}{c|}{\textbf{Classification error}}& \multicolumn{3}{c}{\textbf{CPU time}}\\\hline
\textbf{Algorithms/ Datasets}\!&Movement&~Devices~& ~Activity~&Movement&~Devices~&~Activity~\\ \hline \hline
    M-Forgetron ($B=10$) & $1.60$ &$0.53$  &$1.14$&  $0.92$& $0.26$ & $0.53$ \\ \hline
    M-Forgetron ($B=50$)& $1.64$ & $0.53$ &$1.14$&$0.92$&$0.27$ & $0.53$  \\ \hline
    M-Forgetron ($B=100$)& $1.42$ & $0.53$&$1.14$ & $0.94$ &$0.29$& $0.53$ \\ \hline        
    Raker ($D=10$) & $26.3$ & $8.37$ &$3.26$& $\mathbf{0.35}$ &$\mathbf{0.10}$& $\mathbf{0.18}$\\ \hline
    Raker ($D=50$) &$9.74$ & $2.54$&$0.58$ &${0.40}$& ${0.12}$ & ${0.21}$\\ \hline         
    Raker ($D=100$) & $4.65$ & $1.53$ &$0.43$&$0.46$&$0.15$ & $0.26$ \\ \hline 
    
	AdaRaker ($D=10$) &$2.46$ & ${0.66}$&$0.68$& $6.13$& $0.65$ & $3.22$\\ \hline 
    AdaRaker ($D=50$)& $\mathbf{1.10}$ & $\mathbf{0.36}$ & $\mathbf{0.34}$&$6.76$ & $1.73$ & $3.56$\\ \hline  
    AdaRaker ($D=100$)& $\mathbf{1.10}$ & $\mathbf{0.36}$ &$\mathbf{0.34}$& $7.55$ & $2.04$& $4.22$\\ \hline
    \end{tabular}
\end{center}
\vspace{-0.2cm}
\caption{Classification error ($\%$) and CPU time (second) versus complexity.}
\label{tab:class2}
\vspace{-0.1cm}
\end{table}

 \noindent\textbf{Classification performance.}
 The classification error $(1/T)\sum_{t=1}^T\max\{0,{\rm sign}(-y_t\hat{y}_t)\}$ and the CPU time of each algorithm on these datasets are  summarized in Table \ref{tab:class1} when a default stepsize $1/\sqrt{T}$ is used for {POLY, LINEAR, RBF, AvgMKL}, OMKL and Raker. The budget of M-Forgetron is set at $B=50$ samples, while Raker and AdaRaker adopt $D=50$ RFs. 
As with the regression tests, it is evident that AdaRaker attains the highest classification accuracy and the Raker has the lowest CPU time among all competing algorithms. Without having to tune stepsizes, the performance of AdaMKL and M-Forgetron is also competitive in this case. 
To explore the best performance of each algorithm, the classification performance under manually tuned stepsizes is reported in Table \ref{tab:class1-opt}, where each algorithm uses the best stepsize among $\{10^{-3},10^{-2},\cdots,10^3\}/\sqrt{T}$ for each dataset. With the optimally chosen stepsizes, the performance of all algorithms improves, and Raker even achieves slightly lower classification error than AdaRaker in some datasets. This is reasonable since compared to Raker with the offline tuned stepsize, AdaRaker will incur some error due to the online adaptation to several (possibly suboptimal) learning rates. 
 
To corroborate the effectiveness of our algorithms in adapting to unknown dynamics (e.g., unknown time horizon $T$ and variability), Table \ref{tab:class:st} compares the performance of AdaRaker with OMKL and Raker using {default, diminishing and optimally tuned stepsizes.} 
Similar to regression tests, the performance of Raker and OMKL is sensitive to the stepsize choice, while AdaRaker achieves the desired performance by combining learners with different learning rates. 
{By simply averaging over all the kernels, AvgMKL outperforms single kernel methods in most cases, but performs much worse than OMKL and (Ada)Raker methods.} 
Note that the Raker also achieves competitive classification accuracy when the constant stepsize $1/\sqrt{T}$ is used. 
Such a choice is however not always feasible in practice, since it requires knowledge of how many data samples will be available ahead of time.

\noindent\textbf{Accuracy versus complexity.} 
In this experiment, we test classification performance in terms of both classification error and CPU time for different levels of complexity; see Table \ref{tab:class2}. 
We use the number of support vectors $B$ for M-Forgetron, and the number of RFs $D$ for Raker and AdaRaker to represent different levels of complexity, and compare their performance using the default stepsize. It is expected that CPU time increases as the complexity increases, and the classification error decreases as the complexity grows. For all three datasets, the AdaRaker achieves the lowest classification error, and the Raker outperforms the M-Forgetron while at the same time it is more efficient computationally.

\section{Concluding Remarks}\label{sec:conc}
This paper dealt with kernel-based learning in environments with unknown dynamics that also include static or slow variations. Uniquely combining advances in random feature based function approximation with online learning from an ensemble of experts, a scalable online multi-kernel learning approach termed Raker, was developed for static environments based on a dictionary of kernels. Endowing Raker with capability of tracking time-varying optimal function estimators, AdaRaker was introduced as an ensemble version of Raker with variable learning rates. The key modules of the novel learning approaches are: i) the random features are for scalability, as they reduce the per-iteration complexity; ii) the preselected kernel dictionary is for flexibility, that is to broaden generalizability of a kernel-based learner over a larger function space; iii) the weighted combination of kernels adjusted online accounts for the reliability of learners; and, iv) the adoption of multiple learning rates is for improved adaptivity to changing environments with unknown dynamics.

Complementing the principled algorithmic design, the performance of Raker is rigorously established using static regret analysis. Furthermore, without a-priori knowledge of dynamics, it is proved that AdaRaker achieves sub-linear dynamic regret, provided that either the loss or the optimal learning function does not change on average.
Experiments on synthetic and real datasets validate the effectiveness of the novel methods.

\acks{This work is supported in part by the National Science Foundation
under Grant 1500713 and 1711471, and NIH 1R01GM104975-01. 
Yanning Shen is also supported by the Doctoral Dissertation Fellowship from
the University of Minnesota.
}

\appendix
\section{Proof of Lemma \ref{lemma4}}\label{app.pf.lemma4}
To prove Lemma \ref{lemma4}, we introduce two intermediate lemmata as follows.  
\begin{lemma}\label{lemma3}
Under (as1), (as2), and $\hat{f}_p^*$ as in \eqref{eq.slot-opt} with ${\cal F}_p:=\{\hat{f}_p|\hat{f}_p(\bbx)=\bbtheta^{\top}\mathbf{z}_p(\bbx),\,\forall \bbtheta\in\mathbb{R}^{2D}\}$, let $\{\hat{f}_{p,t}(\bbx_t)\}$ denote the sequence of estimates generated by Raker with a pre-selected kernel $\kappa_p$. Then the following bound holds true w.p.1
	\begin{align}
	\sum_{t=1}^T {\cal L}_t(\hat{f}_{p,t}(\bbx_t))\!-\!\sum_{t=1}^T{\cal L}_t(\hat{f}_p^*(\bbx_t))\!\leq\! \frac{\|\bbtheta_p^*\|^2}{2\eta}\!+\!\frac{\eta L^2T}{2}
	\end{align}
	where $\eta$ is the learning rate, $L$ is the Lipschitz constant in (as2), and $\bbtheta_p^*$ is the corresponding parameter (or weight) vector supporting the best estimator  $\hat{f}_p^*(\bbx)=(\bbtheta_p^*)^{\top}\mathbf{z}_p(\bbx)$.
\end{lemma}

\begin{proof}
Similar to the regret analysis of online gradient descent \citep{shalev2011}, using \eqref{eq:weit-rf} for any fixed $\bbtheta$, we find 
\begin{align}
\label{eq:sreg:1}
	\|\bbtheta_{p, t+1}-\bbtheta\|^2=&\|\bbtheta_{p,t}-\eta\nabla{\cal L}(\bbtheta_{p,t}^\top\bbz_p(\bbx_t),y_t)-\bbtheta\|^2\\
	=&\|\bbtheta_{p,t}-\bbtheta\|^2+\eta^2\|\nabla{\cal L}(\bbtheta_{p,t}^\top\bbz_p(\bbx_t),y_t)\|^2-2\eta\nabla^{\top}{\cal L}(\bbtheta_{p,t}^\top\bbz_p(\bbx_t),y_t)(\bbtheta_{p,t}-\bbtheta).\nonumber
\end{align}
Meanwhile, the convexity of the loss under (as1) implies that
\begin{align}\label{eq:sreg:cvx}
	{\cal L}(\bbtheta_{p,t}^\top\bbz_p(\bbx_t),y_t)-{\cal L}(\bbtheta^\top\bbz_p(\bbx_t),y_t)\leq\nabla^{\top}{\cal L}(\bbtheta_{p,t}^\top\bbz_p(\bbx_t),y_t)(\bbtheta_{p,t}-\bbtheta).
\end{align}
Plugging \eqref{eq:sreg:cvx} into \eqref{eq:sreg:1} and rearranging terms yields
\begin{equation}
\label{eq:sreg:2}
{\cal L}(\bbtheta_{p,t}^\top\bbz_p(\bbx_t),y_t)\!-\!{\cal L}(\bbtheta^\top\bbz_p(\bbx_t),y_t)\leq \frac{\|\bbtheta_{p,t}-\bbtheta\|^2-\|\bbtheta_{p,t+1}-\bbtheta\|^2}{2\eta}+\frac{\eta}{2}\|\nabla{\cal L}(\bbtheta_{p,t}^\top\bbz_p(\bbx_t),y_t)\|^2.
\end{equation}
Summing \eqref{eq:sreg:2} over $t=1,\dots,T$, with $\hat{f}_{p,t}(\bbx_t)=\bbtheta_{p,t}^\top\bbz_p(\bbx_t)$, we arrive at
\begin{align}\label{eq:sreg:3}
	&\sum_{t=1}^T\left({\cal L}(\hat{f}_{p,t}(\bbx_t),y_t)\!-\!{\cal L}(\bbtheta^\top\bbz_p(\bbx_t),y_t)\right)\nonumber\\
	\leq &\, \frac{\|\bbtheta_{p,1}\!-\!\bbtheta\|^2-\|\bbtheta_{p,T+1}-\bbtheta\|^2}{2\eta}+\frac{\eta}{2}\sum_{t=1}^T\|\nabla{\cal L}(\bbtheta_{p,t}^\top\bbz_p(\bbx_t),y_t)\|^2\nonumber\\
	\stackrel{(a)}{\leq} &\, \frac{\|\bbtheta\|^2}{2\eta}+\frac{\eta L^2 T}{2}
\end{align}
where (a) uses the Lipschitz constant in (as2), the non-negativity of $\|\bbtheta_{p,T+1}-\bbtheta\|^2$, and the initial value $\bbtheta_{p,1}=\mathbf{0}$. The proof of Lemma \ref{lemma3} is then complete by choosing $\bbtheta=\bbtheta_p^*=\sum_{t=1}^T \alpha_{p,t}^* \bbz_p(\bbx_t)$ such that $\hat{f}_p^*(\bbx_t)=\bbtheta^\top\bbz_p(\bbx_t)$ in \eqref{eq:sreg:3}.
\end{proof}
Lemma \ref{lemma3} establishes that the static regret of the Raker is upper bounded by some constants, which mainly depend on the stepsize in \eqref{eq.klp-weight} and the time horizon $T$.

In addition, we will bound the difference between the loss of the solution obtained from Algorithm \ref{algo:omkl:rf} and the loss of the best single kernel-based online learning algorithm. Specifically the following lemma holds:
\begin{lemma}
\label{lemma10}
Under (as1) and (as2), with $\{\hat{f}_{p,t}\}$ generated from Raker, it holds that 
	\begin{equation}
	\label{eq:lemm10}
		\sum_{t=1}^T \sum_{p=1}^P \bar{w}_{p,t} {\cal L}_{t}(\hat{f}_{p,t}(\bx_t))- \sum_{t=1}^T {\cal L}_{t}(\hat{f}_{p,t}(\bx_t))\leq\eta T+\frac{\ln P}{\eta}
	\end{equation}
	where $\eta$ is the learning rate in \eqref{eq.mkl-weight}, and $P$ is the number of kernels in the dictionary. 
\end{lemma}

\begin{proof}
Letting $W_{t}:=\sum_{p=1}^P w_{p,t}$, the weight recursion in \eqref{eq.mkl-weight} implies that
\begin{eqnarray}
\label{eq:sreg:W1}
	W_{t+1}&=&\sum_{p=1}^P w_{p,t+1}=\sum_{p=1}^P w_{p,t} \exp\left(-\eta{\cal L}_t\left(\hat{f}_{p,t}(\bx_t)\right)\right)\nonumber\\
	&\leq&\sum_{p=1}^P w_{p,t}\left(1-\eta{\cal L}_t\left(\hat{f}_{p,t}(\bx_t)\right)+\eta^2{\cal L}_t\left(\hat{f}_{p,t}(\bx_t)\right)^2\right)
\end{eqnarray}
where the last inequality holds because $\exp(-\eta x)\leq 1-\eta x+\eta^2 x^2$, for $|\eta|\leq 1$.
Furthermore, substituting $\bar{w}_{p,t}:=w_{p,t}/\sum_{p=1}^P w_{p,t}=w_{p,t}/W_t$ into \eqref{eq:sreg:W1}, it follows that
\begin{align}
	\label{eq:sreg:W2-0}
	W_{t+1}&\leq \sum_{p=1}^P W_t\bar{w}_{p,t}\left(1-\eta{\cal L}_t\left(\hat{f}_{p,t}(\bx_t)\right)+\eta^2{\cal L}_t\left(\hat{f}_{p,t}(\bx_t)\right)^2\right)\nonumber\\
	&= W_t\Bigg(1-\eta\sum_{p=1}^P\bar{w}_{p,t} {\cal L}_t\left(\hat{f}_{p,t}(\bx_t)\right)+\eta^2\sum_{p=1}^P\bar{w}_{p,t} {\cal L}_t\left(\hat{f}_{p,t}(\bx_t)\right)^2\Bigg).
\end{align}
Using $1+x\leq e^x,\,\forall x$, \eqref{eq:sreg:W2-0} leads to
\begin{align}
	\label{eq:sreg:W2}
	W_{t+1}&\leq W_t \exp \Bigg(-\eta \sum_{p=1}^P\bar{w}_{p,t} {\cal L}_t\left(\hat{f}_{p,t}(\bx_t)\right)+\eta^2 \sum_{p=1}^P\bar{w}_{p,t} {\cal L}_t\left(\hat{f}_{p,t}(\bx_t)\right)^2\Bigg).
\end{align}
Telescoping \eqref{eq:sreg:W2} from $t=1$ to $T$, we have ($W_1=1$)
\begin{align}\label{eq:sreg:W3}
	W_{T+1}&\leq \exp \Bigg(-\eta \sum_{t=1}^T\sum_{p=1}^P\bar{w}_{p,t} {\cal L}_t\left(\hat{f}_{p,t}(\bx_t)\right)+\eta^2  \sum_{t=1}^T\sum_{p=1}^P\bar{w}_{p,t} {\cal L}_t\left(\hat{f}_{p,t}(\bx_t)\right)^2\Bigg).
\end{align}

On the other hand, for any $p$, the following holds true
\begin{eqnarray}
\label{eq:sreg:W4}
	W_{T+1}\geq w_{p,T+1}&=&w_{p,1}\prod_{t=1}^T \exp(-\eta{\cal L}_t\left(\hat{f}_{p,t}(\bx_t)\right))\nonumber\\
	&= & w_{p,1}\exp\Bigg(-\eta\sum_{t=1}^T{\cal L}_t\left(\hat{f}_{p,t}(\bx_t)\right)\Bigg).
\end{eqnarray}
Combining \eqref{eq:sreg:W3} with \eqref{eq:sreg:W4}, we arrive at
\begin{align}
\label{eq:sreg:6} 
&\exp \Bigg(\!-\!\eta \sum_{t=1}^T\sum_{p=1}^P\bar{w}_{p,t} {\cal L}_t\left(\hat{f}_{p,t}(\bx_t)\right)+\eta^2  \sum_{t=1}^T\sum_{p=1}^P\bar{w}_{p,t} {\cal L}_t\left(\hat{f}_{p,t}(\bx_t)\right)^2\!\Bigg)\nonumber\\
	\geq\, & w_{p,1} \exp\Bigg(\!-\!\eta\sum_{t=1}^T{\cal L}_t\left(\hat{f}_{p,t}(\bx_t)\right)\!\Bigg).
\end{align}
Taking the logarithm on both sides of \eqref{eq:sreg:6}, we find that (cf. $w_{p,1}=1/P$)
\begin{eqnarray}
\label{eq:sreg:7}
-\eta \sum_{t=1}^T\sum_{p=1}^P\bar{w}_{p,t} {\cal L}_t\!\left(\hat{f}_{p,t}(\bx_t)\right)\!+\eta^2  \sum_{t=1}^T\sum_{p=1}^P\bar{w}_{p,t} {\cal L}_t\!\left(\hat{f}_{p,t}(\bx_t)\right)^2
	\!\geq\!-\eta\sum_{t=1}^T{\cal L}_t\!\left(\hat{f}_{p,t}(\bx_t)\right)\!-\ln P
\end{eqnarray}
which leads to 
\begin{equation}\label{eq:sreg:8}
	\sum_{t=1}^T\sum_{p=1}^P\bar{w}_{p,t} {\cal L}_t\left(\hat{f}_{p,t}(\bx_t)\right)\leq \sum_{t=1}^T{\cal L}_t\left(\hat{f}_{p,t}(\bx_t)\right)+\eta  \sum_{t=1}^T\sum_{p=1}^P\bar{w}_{p,t} {\cal L}_t\left(\hat{f}_{p,t}(\bx_t)\right)^2+\frac{\ln P}{\eta}
	\end{equation}
and the proof is complete since ${\cal L}_t\left(\hat{f}_{p,t}(\bx_t)\right)^2\leq 1$ and $\sum_{p=1}^P\bar{w}_{p,t}=1$.
\end{proof} 

Moreover, since ${\cal L}_t(\cdot)$ is convex under (as1), Jensen's inequality implies that
\begin{align}
\label{eq:sreg:9}
{\cal L}_t\bigg(\sum_{p=1}^P \bar{w}_{p,t} \hat{f}_{p,t}(\bx_t)\bigg)\leq \sum_{p=1}^P \bar{w}_{p,t} {\cal L}_t\left(\hat{f}_{p,t}(\bx_t)\right).
\end{align}
Plugging \eqref{eq:sreg:9} into \eqref{eq:lemm10} in Lemma \ref{lemma10}, we arrive at 
	\begin{align}
		\sum_{t=1}^T{\cal L}_t\bigg(\sum_{p=1}^P \bar{w}_{p,t} \hat{f}_{p,t}(\bx_t)\bigg)\leq &\sum_{t=1}^T {\cal L}_{t}\left(\hat{f}_{p,t}(\bx_t)\right)+\eta T+\frac{\ln P}{\eta}\nonumber\\
	\stackrel{(a)}{\leq}  &\sum_{t=1}^T{\cal L}_t\left(\hat{f}_p^*(\bbx_t)\right)+\frac{\ln P}{\eta}+\frac{\|\bbtheta_p^*\|^2}{2\eta}+\frac{\eta L^2T}{2}+\eta T
	\end{align}
where (a) follows because $\bbtheta_p^*$ is the optimal solution for any given kernel $\kappa_p$. This proves the claim in Lemma \ref{lemma4}.

\section{Proof of Theorem \ref{theorem0}}\label{app.pf.theorem0}
To derive the performance bound relative to the best function estimator $f^*(\bbx_t)$ in the RKHS, the key step is to bound the error of approximation. 
For a given shift-invariant $\kappa_p$, the maximum point-wise error of the RF kernel approximant is uniformly bounded with probability at least
$
	1-2^8\big(\frac{\sigma_p}{\epsilon}\big)^2 \exp \big(\frac{-D\epsilon^2}{4d+8}\big), 
$ by \citep{rahimi2007}
\begin{align}
\label{ieq:1}
	\sup_{\bbx_i,\bbx_j\in{\cal X}} \left|\bbz_p^\top(\bbx_i)\bbz_p(\bbx_j)-\kappa_p(\bbx_i,\bbx_j)\right|<\epsilon 
\end{align}
where $\epsilon>0$ is a given constant, $D$ the number of features, while $d$ represents the dimension of $\bbx$, and $\sigma_p^2:=\mathbb{E}_p[\|\bbv\|^2]$ is the second-order moments of the RF vector norm.  
Henceforth, for the optimal function estimator \eqref{eq.slot-opt} in ${\cal H}_p$ denoted by $f_p^*(\bbx):=\sum_{t=1}^T\alpha_{p,t}^* \kappa_p(\bbx,\bbx_t)$, and its RF-based approximant $\check{f}^*_p:=\sum_{t=1}^T\alpha_{p,t}^* \bbz_p^\top(\bbx)\bbz_p(\bbx_t)\in{\cal F}_p$, we have
\begin{align}
\label{eq:sreg:5.a}
	\left|\sum_{t=1}^T{\cal L}_t\left(\check{f}^*_p(\bbx_t)\right)-\sum_{t=1}^T{\cal L}_t\left(f_p^*(\bbx_t)\right)\right|\stackrel{(a)}{\leq}\, &\sum_{t=1}^T\left|{\cal L}_t\left(\check{f}^*_p(\bbx_t)\right)-{\cal L}_t(f_p^*(\bbx_t))\right|\nonumber\\
\stackrel{(b)}{\leq}\, &\sum_{t=1}^T L\left|\sum_{t'=1}^T\alpha_{p,t'}^*\bbz_p^\top(\bbx_{t'})\bbz_p(\bbx_t)-\sum_{t'=1}^T\alpha_{p,t'}^*\kappa_p(\bbx_{t'},\bbx_t)\right|\nonumber\\
\stackrel{(c)}{\leq}\, &\sum_{t=1}^T L\sum_{t'=1}^T|\alpha_{p,t'}^*|\left|\bbz_p^\top(\bbx_{t'})\bbz_p(\bbx_t)-\kappa_p(\bbx_{t'},\bbx_t)\right|
\end{align}
where (a) follows from the triangle inequality; (b) uses the Lipschitz continuity of the loss, and (c) is due to the Cauchy-Schwarz inequality.  
Combining with \eqref{ieq:1}, yields
\begin{eqnarray}
\label{eq:sreg:5}
\left|\sum_{t=1}^T{\cal L}_t(\check{f}^*_p(\bbx_t))-\sum_{t=1}^T{\cal L}_t(f_p^*(\bbx_t))\right|\leq \sum_{t=1}^T L\epsilon\sum_{t'=1}^T |\alpha_{p,t'}^*|\leq \epsilon L T C,~{\rm w.h.p.}
\end{eqnarray}
where the equality follows from $C:=\max_p \sum_{t=1}^T |\alpha_{p,t}^*|$.
Under the kernel bounds in (as3), the uniform convergence in \eqref{ieq:1} implies that $\sup_{\bbx_t,\bbx_{t'}\in{\cal X}} \bbz_p^\top(\bbx_t)\bbz_p(\bbx_{t'})\leq 1+\epsilon$, w.h.p., which in turn leads to
\begin{align}
\label{eq:sreg:4}
	\left\|\bbtheta_p^*\right\|^2:=\left\|\sum_{t=1}^T\alpha_{p,t}^*\bbz_p(\bbx_t)\right\|^2=\left|\sum_{t=1}^T\sum_{t'=1}^T\alpha_{p,t}^*\alpha_{p,t'}^*\bbz_p^{\top}(\bbx_t)\bbz_p(\bbx_{t'})\right|\leq (1+\epsilon)C^2
\end{align}
where we again used the definition of $C$. 

Lemma \ref{lemma4} together with \eqref{eq:sreg:5} and \eqref{eq:sreg:4} leads to the regret of the proposed Raker algorithm relative to the best static function in ${\cal H}_p$, that is given by 
	\begin{align}
		&\sum_{t=1}^T{\cal L}_t\bigg(\sum_{p=1}^P w_{p,t} \hat{f}_{p,t}(\bbx_t)\bigg)-\sum_{t=1}^T{\cal L}_t(f_p^*(\bbx_t))\nonumber\\
		=&\sum_{t=1}^T{\cal L}_t\bigg(\sum_{p=1}^P w_{p,t} \hat{f}_{p,t}(\bbx_t)\bigg)-\sum_{t=1}^T{\cal L}_t\left(\check{f}^*_p(\bbx_t)\right)+\sum_{t=1}^T{\cal L}_t\left(\check{f}^*_p(\bbx_t)\right)-\sum_{t=1}^T{\cal L}_t(f_p^*(\bbx_t))\nonumber\\
		\leq\,  &\frac{\ln P}{\eta}+\frac{\eta L^2T}{2}+\eta T+ \frac{(1+\epsilon)C^2}{2\eta}+\epsilon L T C,~{\rm w.h.p.}
	\end{align}
which completes the proof of Theorem \ref{theorem0}.

\section{Proof of Lemma \ref{lemma7}}\label{app.pf.lemm7}
Using Theorem \ref{theorem0} with $\eta=\epsilon={\cal O}(1/\sqrt{T})$, it holds w.h.p. that
\begin{align}\label{eq.pf.sreg}
\sum_{t=1}^T{\cal L}_t\bigg(\sum_{p=1}^P w_{p,t} \hat{f}_{p,t}(\bbx_t)\bigg)-\sum_{t=1}^T{\cal L}_t(f_{p^*}^*(\bbx_t))\!\leq\! \left(\ln P+\frac{C^2}{2}+\frac{L^2}{2}+LC\right)\!\sqrt{T}\!:=c_0\sqrt{T}
\end{align}
where $p^*:=\arg\min_{p\in{\cal P}}\sum_{t=1}^T{\cal L}_t\left(\hat{f}_p^*(\bbx_t)\right)$. 
At the end of interval $I$, we then deduce that the static regret of the Raker learner ${\cal A}_I$ is (cf. \eqref{eq.inter-reg})
\begin{align}\label{eq.pf.sreg2}
{\rm Reg}^{\rm s}_{{\cal A}_I}(|I|)=\sum_{t\in I}{\cal L}_t\left(\hat{f}_t^{(I)}(\bbx_t)\right)-\sum_{t\in I}{\cal L}_t(f_I^*(\bbx_t))\!\leq c_0\sqrt{|I|},~{\rm w.h.p.}
\end{align}
where $\hat{f}_t^{(I)}(\bbx_t)$ is defined in \eqref{eq.comb-up}, and $f_I^*\in\arg\min_{f\in\bigcup_{p\in{\cal P}}{\cal H}_p} \sum_{t\in I} {\cal L}_t(f(\bbx_t))$.
To this end, we sketch the main steps leading to Lemma \ref{lemma7} as follows. 

For every interval $I$, the static regret of the AdaRaker can be decomposed as
\begin{align}\label{eq.pf.sreg3}
{\rm Reg}^{\rm s}_{\rm AdaRaker}(|I|)&=\sum_{t\in I}{\cal L}_t\left(\hat{f}_t(\bbx_t)\right)\!-\!\sum_{t\in I}{\cal L}_t\left(\hat{f}_t^{(I)}(\bbx_t)\right)+\sum_{t\in I}{\cal L}_t\left(\hat{f}_t^{(I)}(\bbx_t)\right)\!-\!\sum_{t\in I}{\cal L}_t(f_I^*(\bbx_t))\nonumber\\
:\!&={\cal R}_1+{\cal R}_2
\end{align}
where the first two sums in \eqref{eq.pf.sreg3} represented by ${\cal R}_1$ capture the regret 
of the Ada-Raker learner ${\cal A}$ relative to the Raker learner ${\cal A}_I$; while the last two sums in \eqref{eq.pf.sreg3} forming ${\cal R}_2$ denote the static regret of ${\cal A}_I$ on this interval. 
Notice that ${\cal R}_2$ directly follows from \eqref{eq.pf.sreg2}, while ${\cal R}_1$ can be bounded following the same steps in Lemma \ref{lemma10}.  
Different from the kernel selections however, the crux is that the number of Raker learners (experts) $|{\cal I}(t)|$ is time-varying. 

A tight bound can be derived via the \emph{Sleeping Experts} reformulation of \citep{luo2015,daniely2015}, where an expert that has never appeared is thought of as being \emph{asleep} for all previous rounds. 
For a looser bound, we assume the experts (instances $\{{\cal A}_I\}$) ever appeared until $t$ are all active; that is, the total number of experts is upper bounded by $t\log t$, since at most $\log t$ experts are run during time $t$. Using \eqref{eq:sreg:W1}-\eqref{eq:sreg:8}, we have that
\begin{align}\label{eq.pf.sreg4}
{\cal R}_1\leq  \eta^{(I)} |I|+\frac{\ln (t\log t)}{\eta^{(I)}}=\sqrt{|I|}\left(1+\ln t+\ln(\log t)\right)\leq \sqrt{|I|}\left(1+2\ln t\right)
\end{align}
where $\eta^{(I)}:={1}/{\sqrt{|I|}}$, and $\ln(\log t)\!\leq\!\ln(t)$. With \eqref{eq.pf.sreg2}, for any interval $I\in {\cal I}$, we have 
\begin{align}\label{eq.pf.sreg5}
	{\rm Reg}^{\rm s}_{\rm AdaRaker}(|I|)=\sqrt{|I|}\left(1+c_0+2\ln t\right)\leq \sqrt{|I|}\left(1+c_0+2\ln T\right).
\end{align}
Since the static regret bound \eqref{eq.pf.sreg3} holds only at the end of such interval,
the bound \eqref{eq.pf.sreg5} only holds for those intervals (collected in ${\cal I}$) (re)initializing Raker instance ${\cal A}_I$.

The next step is to show that \eqref{eq.pf.sreg5} holds for any interval $I\subseteq {\cal T}$, possibly $I\notin {\cal I}$. 
This is possible whenever the interval set ${\cal I}$ is properly designed, e.g., the interval partition given in Section 4.1. 
For any interval $I$, define the set of subintervals covered by $I$ as ${\cal I}\|I:=\{I'\subseteq I, I'\in {\cal I}\}$. As argued in \cite[Lemma 5]{daniely2015}, interval $I$ can be partitioned into two sequences of \emph{non-overlapping} but \emph{consecutive} intervals, given by $\{I_{-m},\ldots,I_0\}\subseteq{\cal I}\|I$ and $\{I_{1},\ldots,I_n\}\subseteq{\cal I}\|I$, the lengths of which satisfy $|I_{k+1}|/|I_{k}|\leq 1/2,\,\forall k\in[1,n-1]$ and $|I_{k}|/|I_{k+1}|\leq 1/2,\,\forall k\in[-m,-1]$. Therefore, we have (using $\sum_{k=1}^{\infty}\sqrt{2^{-k}T_0}\leq 4\sqrt{T_0}$)
\begin{align}\label{eq.pf.sreg6}
	{\rm Reg}^{\rm s}_{\rm AdaRaker}(|I|)=\sum_{k=1}^{n-1}{\rm Reg}^{\rm s}_{\rm Ada}(|I_k|)+\sum_{k=-m}^{-1}{\rm Reg}^{\rm s}_{\rm Ada}(|I_k|)\leq C_0\sqrt{|I|}+C_1\ln T\sqrt{|I|}
\end{align}
where the inequality follows from \eqref{eq.pf.sreg5} with $|I|$ replaced by $|I_k|$ $(\leq |I|)$, and $C_0$, $C_1$ are constants depending on $c_0$ defined in \eqref{eq.pf.sreg}. This completes the proof of Lemma \ref{lemma7}.

\section{Proof of Theorem \ref{thm.dyn-reg}}\label{app.pf.thm.dyn-reg}
To start, the dynamic regret in \eqref{eq.dyn-reg} can be decomposed as
\begin{align}\label{eq.pf.dyn-reg}
    {\rm Reg}^{\rm d}_{\cal A}(T):=\sum_{t=1}^T {\cal L}_t(\hat{f}_t(\bbx_t))-\sum_{t=1}^T{\cal L}_t(f^*(\bbx_t))+ \sum_{t=1}^T{\cal L}_t(f^*(\bbx_t))-\sum_{t=1}^T{\cal L}_t(f_t^*(\bbx_t))
\end{align}
where $f^*(\cdot)$ is the best fixed function in \eqref{eq.slot-opt}, and $f_t^*(\cdot)$ is the best dynamic function in \eqref{eq.realtime-prob}, both of which belong to the union of spaces $\bigcup_{p\in{\cal P}}{\cal H}_p$. 
In \eqref{eq.pf.dyn-reg}, the first difference of sums is the static regret of AdaRaker, while the second difference of sums is the relative loss between the best fixed function and the best dynamic solution in the common space.

Intuitively, if the time horizon $T$ is large, then the average static regret will become small, but the gap between the two benchmarks is large. 
With the insights gained from \citep{besbes2015,luo2017}, $T$ essentially trades off the values of two terms. Thus, splitting ${\cal T}$ into sub-horizons $\{{\cal T}_s\}, s=1,\ldots,\lfloor T/\Delta T\rfloor$, each having length $\Delta T$, the dynamic regret of AdaRaker can be bounded by
	\begin{align}\label{pf.dyn-regdep}
		{\rm Reg}^{\rm d}_{\rm AdaRaker}(T)&=\!\!\sum_{s=1}^{\lfloor T/\Delta T\rfloor}\!\!\sum_{t\in {\cal T}_s}\left({\cal L}_t(\hat{f}_t(\bbx_t))\!-\!{\cal L}_t(f_{{\cal T}_s}^*(\bbx_t))\right)\!+\!\!\sum_{s=1}^{\lfloor T/\Delta T\rfloor}\!\! \sum_{t\in {\cal T}_s}\left({\cal L}_t(f_{{\cal T}_s}^*(\bbx_t))\!-\!{\cal L}_t(f_t^*(\bbx_t))\right)\nonumber\\
		:\!&=\sum_{s=1}^{\lfloor T/\Delta T\rfloor}{\cal R}_1+\sum_{s=1}^{\lfloor T/\Delta T\rfloor}{\cal R}_2
		\end{align}
where the first sum over ${\cal T}_s$ we define as ${\cal R}_1$ can be bounded under AdaRaker from Lemma \ref{lemma7}, while the second sum over  ${\cal T}_s$ that we define as ${\cal R}_2$ depends on the variability of the environments ${\cal V}(\{{\cal L}_t\})$, can be bounded by \cite[Prop. 2]{besbes2015}
\begin{align}\label{pf.dyn-regdep2}
	{\cal R}_2\leq 2\Delta T{\cal V}(\{{\cal L}_t\}_{t\in {\cal T}_s}).
\end{align}
Together with Lemma \ref{lemma7}, it follows that
	\begin{align}\label{pf.dyn-regbound}
		{\rm Reg}^{\rm d}_{\rm AdaRaker}(T)& \leq  \sum_{s=1}^{\lfloor T/\Delta T\rfloor}\left((C_0+C_1\ln T)\sqrt{\Delta T}+2\Delta T{\cal V}(\{{\cal L}_t\}_{t\in {\cal T}_s})\right)\nonumber\\
		&=(C_0+C_1\ln T)\frac{T}{\sqrt{|\Delta T|}}+2|\Delta T|{\cal V}(\{{\cal L}_t\}_{t=1}^T),\, {\rm w.h.p.} 
	\end{align}
	Since \eqref{eq.lemma7} in Lemma \ref{lemma7} holds for any interval $\Delta T\subseteq {\cal T}$, after selecting $\Delta T$ so that $|\Delta T|=\left(T/{\cal V}(\{{\cal L}_t\}_{t=1}^T)\right)^{\frac{2}{3}}$, we arrive at 
		\begin{align}\label{pf.dyn-regbound2}
		{\rm Reg}^{\rm d}_{\rm AdaRaker}(T)\leq  (C_0+C_1\ln T)T^{\frac{2}{3}}{\cal V}^{\frac{1}{3}}(\{{\cal L}_t\}_{t=1}^T)+2T^{\frac{2}{3}}{\cal V}^{\frac{1}{3}}(\{{\cal L}_t\}_{t=1}^T),\, {\rm w.h.p.}
	\end{align}
which completes the proof of Theorem \ref{thm.dyn-reg}.
   
\section{Proof of Theorem \ref{thm.reg-sa-d}}\label{app.pf.thm.reg-sa-d}
Suppose that the $m$-switching dynamic solution $\{\check{f}_t^*\}$ changes at slots $t_1=1,\ldots,t_m$, and define the $m$ sub-intervals that partition ${\cal T}:=\{1,\ldots,T\}$ as ${\cal T}_1:=[1,t_2-1]$, ${\cal T}_2:=[t_2,t_3-1], \ldots$, and ${\cal T}_m:=[t_m,T]$.  
To use the bound in Lemma \ref{lemma7}, we decompose the regret of AdaRaker relative to the $m$-switching dynamic solution $\{\check{f}_t^*\}$ by
\begin{align}\label{pf.stw-reg}
	{\rm Reg}^m_{\rm AdaRaker}(T)& \stackrel{(a)}{=}\sum_{s=1}^m \, \sum_{t\in{\cal T}_s} \left({\cal L}_t(\hat{f}_t(\bbx_t))-{\cal L}_t(\check{f}_{t_s}^*(\bbx_t))\right)\nonumber\\
	&\stackrel{(b)}{\leq} \sum_{s=1}^m \, \sum_{t\in{\cal T}_s} \left({\cal L}_t(\hat{f}_t(\bbx_t))-{\cal L}_t(f_{{\cal T}_s}^*(\bbx_t))\right)
\end{align}
where (a) holds because the definition of $\{\check{f}_t^*\}$ in \eqref{eq.m-switch} implies that $\check{f}_t^*=\check{f}_{t_s}^*,\,\forall t\in{\cal T}_s$, and (b) because the best fixed function is given by $f_{{\cal T}_s}^*\in\arg\min_{f\in {f\in\bigcup_{p\in{\cal P}}{\cal H}_p}} \,\sum_{t\in {\cal T}_s} {\cal L}_t(f(\bbx_t))$. 
Therefore, using the regret bound of Lemma \ref{lemma7} in \eqref{eq.lemma7}, we have
\begin{equation}
	{\rm Reg}^m_{\rm AdaRaker}(T)\leq \sum_{s=1}^m {\rm Reg}_{\cal A}^{\rm s} (|{\cal T}_s|)\leq (C_0+C_1\ln T)\sum_{s=1}^m\sqrt{|{\cal T}_s|}.
\end{equation}
Holder's inequality further implies that 
\begin{align}
	{\rm Reg}^m_{\rm AdaRaker}(T)&\leq (C_0+C_1\ln T)\left(\sum_{s=1}^m (1)^2\right)^{\frac{1}{2}}
	\left(\sum_{s=1}^m \left(\sqrt{|{\cal T}_s|}\right)^2\right)^{\frac{1}{2}}\nonumber\\
	&\leq (C_0+C_1\ln T)\sqrt{T m},~{\rm w.h.p.}
\end{align}
which completes the proof of Theorem \ref{thm.reg-sa-d}.

\vspace{3cm}

\bibliographystyle{IEEEtran}
\bibliography{myabrv,net,dmkl}

\begin{thebibliography}{51}
\providecommand{\natexlab}[1]{#1}
\providecommand{\url}[1]{\texttt{#1}}
\expandafter\ifx\csname urlstyle\endcsname\relax
  \providecommand{\doi}[1]{doi: #1}\else
  \providecommand{\doi}{doi: \begingroup \urlstyle{rm}\Url}\fi

\bibitem[Anguita et~al.(2013)Anguita, Ghio, Oneto, Parra, and
  Reyes-Ortiz]{anguita2013}
Davide Anguita, Alessandro Ghio, Luca Oneto, Xavier Parra, and Jorge~Luis
  Reyes-Ortiz.
\newblock A public domain dataset for human activity recognition using
  smartphones.
\newblock In \emph{Euro. Symp. on Artificial Neural Netw., Comp. Intell. and
  Mach. Learn.}, Bruges, Belgium, Apr. 2013.

\bibitem[Bacciu et~al.(2014)Bacciu, Barsocchi, Chessa, Gallicchio, and
  Micheli]{bacciu2014experimental}
Davide Bacciu, Paolo Barsocchi, Stefano Chessa, Claudio Gallicchio, and Alessio
  Micheli.
\newblock An experimental characterization of reservoir computing in ambient
  assisted living applications.
\newblock \emph{Neural Computing and Applications}, 24\penalty0 (6):\penalty0
  1451--1464, May 2014.

\bibitem[Bach(2008)]{bach2008}
Francis~R. Bach.
\newblock Consistency of the group lasso and multiple kernel learning.
\newblock \emph{J. Machine Learning Res.}, 9:\penalty0 1179--1225, Jun. 2008.

\bibitem[Bazerque and Giannakis(2013)]{bazerque2013nonparametric}
Juan~Andres Bazerque and Georgios~B. Giannakis.
\newblock Nonparametric basis pursuit via sparse kernel-based learning: A
  unifying view with advances in blind methods.
\newblock \emph{IEEE Signal Processing Magazine}, 30\penalty0 (4):\penalty0
  112--125, Jul. 2013.

\bibitem[Besbes et~al.(2015)Besbes, Gur, and Zeevi]{besbes2015}
Omar Besbes, Yonatan Gur, and Assaf Zeevi.
\newblock Non-stationary stochastic optimization.
\newblock \emph{Operations Research}, 63\penalty0 (5):\penalty0 1227--1244,
  Sep. 2015.

\bibitem[Bouboulis et~al.(2018)Bouboulis, Chouvardas, and
  Theodoridis]{bouboulis2017}
Pantelis Bouboulis, Symeon Chouvardas, and Sergios Theodoridis.
\newblock Online distributed learning over networks in {RKH} spaces using
  random {F}ourier features.
\newblock \emph{IEEE Trans. Sig. Proc.}, to appear, 2018.

\bibitem[Candanedo et~al.(2017)Candanedo, Feldheim, and
  Deramaix]{candanedo2017data}
Luis~M. Candanedo, V{\'e}ronique Feldheim, and Dominique Deramaix.
\newblock Data driven prediction models of energy use of appliances in a
  low-energy house.
\newblock \emph{Energy and Buildings}, 140:\penalty0 81--97, 2017.

\bibitem[Cesa-Bianchi and Lugosi(2006)]{cesa2006}
Nicolo Cesa-Bianchi and G{\'a}bor Lugosi.
\newblock \emph{{Prediction, Learning, and Games}}.
\newblock Cambridge University Press, Cambridge, United Kingdom, 2006.

\bibitem[Cortes et~al.(2009)Cortes, Mohri, and Rostamizadeh]{cortes2009}
Corinna Cortes, Mehryar Mohri, and Afshin Rostamizadeh.
\newblock $\ell_2$-regularization for learning kernels.
\newblock In \emph{Proc. Conf. on Uncertainty in Artificial Intelligence},
  pages 109--116, Montreal, Canada, Jun. 2009.

\bibitem[Cortes et~al.(2010)Cortes, Mohri, and Talwalkar]{cortes2010}
Corinna Cortes, Mehryar Mohri, and Ameet Talwalkar.
\newblock On the impact of kernel approximation on learning accuracy.
\newblock In \emph{Proc. Intl. Conf. on Artificial Intelligence and
  Statistics}, pages 113--120, Sardinia, Italy, May 2010.

\bibitem[Dai et~al.(2014)Dai, Xie, He, Liang, Raj, Balcan, and Song]{dai2014}
Bo~Dai, Bo~Xie, Niao He, Yingyu Liang, Anant Raj, Maria-Florina~F. Balcan, and
  Le~Song.
\newblock Scalable kernel methods via doubly stochastic gradients.
\newblock In \emph{Proc. Advances in Neural Info. Process. Syst.}, pages
  3041--3049, Montreal, Canada, Dec. 2014.

\bibitem[Dai et~al.(2017)Dai, He, Pan, Boots, and Song]{dai2017}
Bo~Dai, Niao He, Yunpeng Pan, Byron Boots, and Le~Song.
\newblock Learning from conditional distributions via dual embeddings.
\newblock In \emph{Proc. Intl. Conf. on Artificial Intelligence and
  Statistics}, pages 1458--1467, Fort Lauderdale, FL, Apr. 2017.

\bibitem[Daniely et~al.(2015)Daniely, Gonen, and Shalev-Shwartz]{daniely2015}
Amit Daniely, Alon Gonen, and Shai Shalev-Shwartz.
\newblock Strongly adaptive online learning.
\newblock In \emph{Proc. Intl. Conf. on Machine Learning}, pages 1405--1411,
  Lille, France, Jun. 2015.

\bibitem[De~Vito et~al.(2008)De~Vito, Massera, Piga, Martinotto, and
  Di~Francia]{de2008field}
Saverio De~Vito, Ettore Massera, M~Piga, L~Martinotto, and G~Di~Francia.
\newblock On field calibration of an electronic nose for benzene estimation in
  an urban pollution monitoring scenario.
\newblock \emph{Sensors and Actuators B: {C}hemical}, 129\penalty0
  (2):\penalty0 750--757, Feb. 2008.

\bibitem[Dekel et~al.(2008)Dekel, Shalev-Shwartz, and Singer]{dekel2008}
Ofer Dekel, Shai Shalev-Shwartz, and Yoram Singer.
\newblock The forgetron: A kernel-based perceptron on a budget.
\newblock \emph{SIAM J. Computing}, 37\penalty0 (5):\penalty0 1342--1372, Jan.
  2008.

\bibitem[Ding et~al.(2017)Ding, Liu, Zhao, and Hoi]{ding2017large}
Yi~Ding, Chenghao Liu, Peilin Zhao, and Steven~CH Hoi.
\newblock Large scale kernel methods for online auc maximization.
\newblock In \emph{Proc. IEEE Intl. Conf. Data Mining}, pages 91--100, New
  Orleans, LO, November 2017.

\bibitem[G{\"o}nen and Alpayd{\i}n(2011)]{gonen2011}
Mehmet G{\"o}nen and Ethem Alpayd{\i}n.
\newblock Multiple kernel learning algorithms.
\newblock \emph{J. Machine Learning Res.}, 12:\penalty0 2211--2268, Jul. 2011.

\bibitem[Hazan(2016)]{hazan2016}
Elad Hazan.
\newblock Introduction to online convex optimization.
\newblock \emph{Found. and Trends in Mach. Learn.}, 2\penalty0 (3-4):\penalty0
  157--325, 2016.

\bibitem[Hoi et~al.(2013)Hoi, Jin, Zhao, and Yang]{hoi2013}
Steven~CH. Hoi, Rong Jin, Peilin Zhao, and Tianbao Yang.
\newblock Online multiple kernel classification.
\newblock \emph{Machine Learning}, 90\penalty0 (2):\penalty0 289--316, Feb.
  2013.

\bibitem[Jadbabaie et~al.(2015)Jadbabaie, Rakhlin, Shahrampour, and
  Sridharan]{jadbabaie2015}
Ali Jadbabaie, Alexander Rakhlin, Shahin Shahrampour, and Karthik Sridharan.
\newblock Online optimization: Competing with dynamic comparators.
\newblock In \emph{Intl. Conf. Artificial Intell. and Stat.}, San Diego, CA,
  May 2015.

\bibitem[Jin et~al.(2010)Jin, Hoi, and Yang]{jin2010}
Rong Jin, Steven~CH. Hoi, and Tianbao Yang.
\newblock Online multiple kernel learning: Algorithms and mistake bounds.
\newblock In \emph{Proc. Intl. Conf. on Algorithmic Learning Theory}, pages
  390--404, Canberra, Australia, Oct. 2010.

\bibitem[Kawala et~al.(2013)Kawala, Douzal-Chouakria, Gaussier, and
  Dimert]{kawala2013predictions}
Fran{\c{c}}ois Kawala, Ahlame Douzal-Chouakria, Eric Gaussier, and Eustache
  Dimert.
\newblock Pr{\'e}dictions d'activit{\'e} dans les r{\'e}seaux sociaux en ligne.
\newblock In \emph{4i{\`e}me Conf{\'e}rence sur les Mod{\`e}les et l'Analyse
  des R{\'e}seaux: Approches Math{\'e}matiques et Informatiques}, 2013.

\bibitem[Kivinen et~al.(2004)Kivinen, Smola, and Williamson]{kivinen2004}
Jyrki Kivinen, Alexander~J. Smola, and Robert~C. Williamson.
\newblock Online learning with kernels.
\newblock \emph{IEEE Trans. Sig. Proc.}, 52\penalty0 (8):\penalty0 2165--2176,
  Aug. 2004.

\bibitem[Lanckriet et~al.(2004)Lanckriet, Cristianini, Bartlett, Ghaoui, and
  Jordan]{lanckriet2004}
Gert~R.G. Lanckriet, Nello Cristianini, Peter Bartlett, Laurent~El Ghaoui, and
  Michael~I. Jordan.
\newblock Learning the kernel matrix with semidefinite programming.
\newblock \emph{J. Machine Learning Res.}, 5:\penalty0 27--72, Jan. 2004.

\bibitem[Lichman(2013)]{Lichman:2013}
Moshe Lichman.
\newblock {UCI} machine learning repository, 2013.
\newblock URL \url{http://archive.ics.uci.edu/ml}.

\bibitem[Lines et~al.(2011)Lines, Bagnall, Caiger-Smith, and
  Anderson]{lines2011classification}
Jason Lines, Anthony Bagnall, Patrick Caiger-Smith, and Simon Anderson.
\newblock Classification of household devices by electricity usage profiles.
\newblock In \emph{Intl. Conf. on Intelligent Data Engineering and Automated
  Learning}, pages 403--412, Norwich, United Kingdom, Sept. 2011.

\bibitem[Lu et~al.(2016)Lu, Hoi, Wang, Zhao, and Liu]{lu2016large}
Jing Lu, Steven~CH. Hoi, Jialei Wang, Peilin Zhao, and Zhi-Yong Liu.
\newblock Large scale online kernel learning.
\newblock \emph{J. Machine Learning Res.}, 17\penalty0 (47):\penalty0 1--43,
  Apr. 2016.

\bibitem[Lu et~al.(2018)Lu, Sahoo, Zhao, and Hoi]{lu2018}
Jing Lu, Doyen Sahoo, Peilin Zhao, and Steven~CH Hoi.
\newblock Sparse passive-aggressive learning for bounded online kernel methods.
\newblock \emph{ACM Trans. Intell. Syst. Tech.}, 9\penalty0 (4):\penalty0 45,
  February 2018.

\bibitem[Luo and Schapire(2015)]{luo2015}
Haipeng Luo and Robert~E. Schapire.
\newblock Achieving all with no parameters: Adanormalhedge.
\newblock In \emph{Proc. Conf. on Learning Theory}, pages 1286--1304, Lille,
  France, Jul. 2015.

\bibitem[Luo et~al.(2017)Luo, Agarwal, and Langford]{luo2017}
Haipeng Luo, Alekh Agarwal, and John Langford.
\newblock Efficient contextual bandits in non-stationary worlds.
\newblock \emph{arXiv preprint:1708.01799}, Aug. 2017.

\bibitem[Ma et~al.(2009)Ma, Saul, Savage, and Voelker]{ma2009}
Justin Ma, Lawrence~K. Saul, Stefan Savage, and Geoffrey~M. Voelker.
\newblock Identifying suspicious {URLs}: {An} application of large-scale online
  learning.
\newblock In \emph{Proc. Intl. Conf. Mach. Learn.}, pages 681--688, Montreal,
  Canada, Jun. 2009.

\bibitem[Micchelli and Pontil(2005)]{micchelli2005}
Charles~A. Micchelli and Massimiliano Pontil.
\newblock Learning the kernel function via regularization.
\newblock \emph{J. Machine Learning Res.}, 6:\penalty0 1099--1125, Jul. 2005.

\bibitem[Rahimi and Recht(2007)]{rahimi2007}
Ali Rahimi and Benjamin Recht.
\newblock Random features for large-scale kernel machines.
\newblock In \emph{Proc. Advances in Neural Info. Process. Syst.}, pages
  1177--1184, Vancouver, Canada, Dec. 2007.

\bibitem[Rakotomamonjy et~al.(2008)Rakotomamonjy, Bach, Canu, and
  Grandvalet]{rakotomamonjy2008}
Alain Rakotomamonjy, Francis~R. Bach, St{\'e}phane Canu, and Yves Grandvalet.
\newblock Simple{MKL}.
\newblock \emph{J. Machine Learning Res.}, 9:\penalty0 2491--2521, Nov. 2008.

\bibitem[Richard et~al.(2009)Richard, Bermudez, and Honeine]{richard2009}
C{\'e}dric Richard, Jos{\'e} Carlos~M. Bermudez, and Paul Honeine.
\newblock Online prediction of time series data with kernels.
\newblock \emph{IEEE Trans. Sig. Proc.}, 57\penalty0 (3):\penalty0 1058--1067,
  Mar. 2009.

\bibitem[Rudi and Rosasco(2017)]{rudi2017}
Alessandro Rudi and Lorenzo Rosasco.
\newblock Generalization properties of learning with random features.
\newblock In \emph{Proc. Advances in Neural Info. Process. Syst.}, pages
  3218--3228, Long Beach, CA, Dec. 2017.

\bibitem[Sahoo et~al.(2014)Sahoo, Hoi, and Li]{sahoo2014online}
Doyen Sahoo, Steven~CH. Hoi, and Bin Li.
\newblock Online multiple kernel regression.
\newblock In \emph{Proc. Intl. Conf. Knowledge Discovery and Data Mining},
  pages 293--302, New York, NY, Aug. 2014.

\bibitem[Sahoo et~al.(2016)Sahoo, Hoi, and Zhao]{sahoo2016}
Doyen Sahoo, Steven~CH. Hoi, and Peilin Zhao.
\newblock Cost sensitive online multiple kernel classification.
\newblock In \emph{Proc. Asian Conf. Machine Learning}, pages 65--80, Hamilton,
  New Zealand, Nov. 2016.

\bibitem[Sch{\"o}lkopf and Smola(2002)]{scholkopf2002}
Bernhard Sch{\"o}lkopf and Alexander~J. Smola.
\newblock \emph{Learning with {K}ernels}.
\newblock MIT Press, Cambridge, MA, 2002.

\bibitem[Shalev-Shwartz(2011)]{shalev2011}
Shai Shalev-Shwartz.
\newblock Online learning and online convex optimization.
\newblock \emph{Found. and Trends in Mach. Learn.}, 4\penalty0 (2):\penalty0
  107--194, 2011.

\bibitem[Shawe-Taylor and Cristianini(2004)]{shawe2004}
John Shawe-Taylor and Nello Cristianini.
\newblock \emph{{Kernel Methods for Pattern Analysis}}.
\newblock Cambridge University Press, Cambridge, United Kingdom, 2004.

\bibitem[Sheikholeslami et~al.(2018)Sheikholeslami, Berberidis, and
  Giannakis]{sheikholeslami2016}
Fateme Sheikholeslami, Dimitris Berberidis, and Georgios~B. Giannakis.
\newblock Large-scale kernel-based feature extraction via budgeted nonlinear
  subspace tracking.
\newblock \emph{IEEE Trans. Sig. Proc.}, 66\penalty0 (8):\penalty0 1967--1981,
  April 2018.

\bibitem[Shen and Giannakis(2018)]{shen2018dsw}
Yanning Shen and Georgios~B Giannakis.
\newblock Online identification of directional graph topologies capturing
  dynamic and nonlinear dependencies.
\newblock In \emph{Proc. of IEEE Data Science Workshop}, pages 195--199,
  Lausanne, Switzerland, June 2018.

\bibitem[Shen et~al.(2016)Shen, Baingana, and Giannakis]{shen2016tmi}
Yanning Shen, Brian Baingana, and Georgios~B. Giannakis.
\newblock Nonlinear structural vector autoregressive models for inferring
  effective brain network connectivity.
\newblock Oct. 2016.
\newblock URL \url{https://arxiv.org/abs/1610.06551}.

\bibitem[Shen et~al.(2017)Shen, Baingana, and Giannakis]{shen2017tsp}
Yanning Shen, Brian Baingana, and Georgios~B. Giannakis.
\newblock Kernel-based structural equation models for topology identification
  of directed networks.
\newblock \emph{IEEE Trans. Sig. Proc.}, 65\penalty0 (10):\penalty0 2503--2516,
  May 2017.

\bibitem[Shen et~al.(2018)Shen, Chen, and Giannakis]{shen2018aistats}
Yanning Shen, Tianyi Chen, and Georgios~B. Giannakis.
\newblock Online ensemble multi-kernel learning adaptive to non-stationary and
  adversarial environments.
\newblock In \emph{Proc. of Intl. Conf. on Artificial Intelligence and
  Statistics}, Lanzarote, Canary Islands, April 2018.

\bibitem[Vovk(1995)]{vovk1995}
Vladimir~G. Vovk.
\newblock A game of prediction with expert advice.
\newblock In \emph{Proc. Annual Conf. Computational Learning Theory}, pages
  51--60, Santa Cruz, CA, Jul. 1995.

\bibitem[Wahba(1990)]{wahba1990spline}
Grace Wahba.
\newblock \emph{Spline {M}odels for {O}bservational {D}ata}.
\newblock SIAM, Philadelphia, PA, 1990.

\bibitem[Wang et~al.(2012)Wang, Crammer, and Vucetic]{wang2012}
Zhuang Wang, Koby Crammer, and Slobodan Vucetic.
\newblock Breaking the curse of kernelization: {B}udgeted stochastic gradient
  descent for large-scale svm training.
\newblock \emph{J. Machine Learning Res.}, 13:\penalty0 3103--3131, Oct. 2012.

\bibitem[Williams and Seeger(2001)]{williams2001}
Christopher~K.I. Williams and Matthias Seeger.
\newblock Using the {N}ystr{\"o}m method to speed up kernel machines.
\newblock In \emph{Proc. Advances in Neural Info. Process. Syst.}, pages
  682--688, Vancouver, Canada, Dec. 2001.

\bibitem[Yu et~al.(2016)Yu, Suresh, Choromanski, Holtmann-Rice, and
  Kumar]{felix2016}
Felix Yu, Ananda~Theertha Suresh, Krzysztof Choromanski, Daniel Holtmann-Rice,
  and Sanjiv Kumar.
\newblock Orthogonal random features.
\newblock In \emph{Proc. Advances in Neural Info. Process. Syst.}, pages
  1975--1983, Barcelona, Spain, Dec. 2016.

\end{thebibliography}
\end{document}